\documentclass{article}

\PassOptionsToPackage{numbers,compress}{natbib}
\usepackage[preprint]{neurips_2026}

\usepackage[utf8]{inputenc}
\usepackage[T1]{fontenc}
\usepackage{hyperref}
\usepackage{url}
\usepackage{booktabs}
\usepackage{amsfonts}
\usepackage{amssymb}
\usepackage{amsmath}
\usepackage{amsthm}
\usepackage{bm}
\usepackage{nicefrac}
\usepackage{microtype}
\usepackage{xcolor}
\usepackage{graphicx}
\usepackage{enumerate}
\usepackage{wrapfig}
\usepackage{subcaption}
\usepackage[linesnumbered,ruled,lined,boxed,commentsnumbered,vlined]{algorithm2e}

\newtheorem{theorem}{Theorem}[section]

\title{Stabilizing the Dynamic Low-Rank Training}

\author{%
  Zhonghan Xu \quad Ling Wang \quad Junhao Chen \quad Jianwei Zhao \quad Jinwei Yang \\
  School of Communications and Information Engineering \\
  University of Electronic Science and Technology of China \\
}

\begin{document}

\maketitle

\begin{abstract}
Training neural networks directly in a low-rank parameterization is an appealing route to reducing memory, compute, and storage simultaneously during both training and inference. Dynamic low-rank training (DLRT), which confines weights to a rank-$r$ manifold via the Galerkin projection of the gradient flow, is particularly attractive because it identifies efficient subnetworks on the fly without specialized initialization or post-factorization. However, DLRT fails to find trainable networks under high compression. In this paper, we derive the gradient flow of the best rank-$r$ approximation and point out that the offset of DLRT comes from a curvature-coupling term which is large and thus non-negligible under aggressive compression. Guided by this analysis, we propose a stable dynamic low-rank training method, named SDLRT, which maintains a lightweight compensation buffer that reinjects the top neglected singular directions. Additionally, we introduce a negative feedback on the truncation tolerance to stabilize each layer's rank. Experimentally, SDLRT reliably finds trainable subnetworks where DLRT collapses and as a PEFT adapter on DeBERTa-v3, it achieves the best average score on SuperGLUE at only $2.8\%$ parameter overhead over LoRA.
\end{abstract}

\section{Introduction}
\label{sec:intro}

Modern neural networks deliver state-of-the-art performance across a wide range of applications, but their ever-growing parameter counts make deployment on memory- and compute-constrained devices increasingly difficult. A natural remedy is to exploit the empirical observation that the weight matrices of over-parameterized networks are approximately low-rank, and to train the network directly in a low-rank parameterization so that memory, compute, and storage are reduced simultaneously in both the training and inference phases~\cite{gural2021lowrank}.

Among existing low-rank training methods, dynamic low-rank training (DLRT) proposed by~\cite{schotthofer2022low} is particularly attractive: it constrains each weight matrix to a low-rank manifold via the Galerkin condition and updates only the three Singular Value Decomposition (SVD) factors $(U, S, V)$, thereby avoiding specialized initialization and the costly post-factorization step. Unlike iterative magnitude pruning~\cite{frankle2019lottery} or Neural Architecture Search-based compression pipelines~\cite{you2022supertickets}, DLRT identifies low-rank ``winning tickets'' on the fly, with resource usage scaling linearly in the current rank rather than in the ambient dimension.

\begin{figure}[t]
  \centering
  \includegraphics[width=\linewidth]{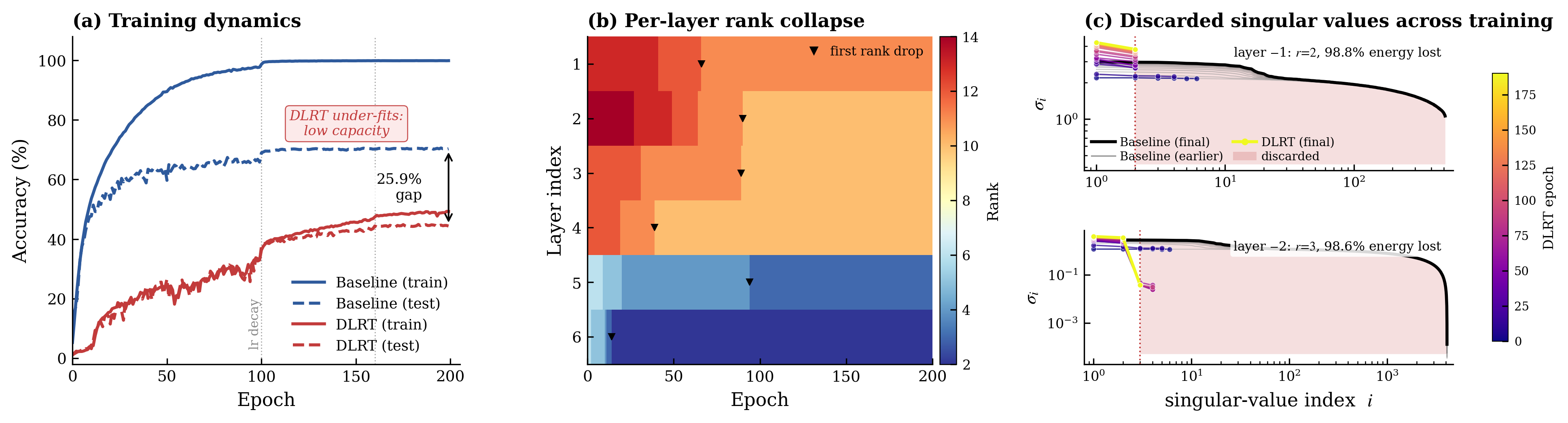}
  \caption{\textbf{DLRT collapses under aggressive compression on VGG-19/CIFAR-100 ($\approx 69.6\%$ compression).} (a) Accuracy curves: DLRT lags the baseline by $25.9\%$ in test accuracy; training accuracy also plateaus far below baseline, indicating insufficient capacity rather than over-fitting. (b) Per-layer rank evolution; $\blacktriangledown$ marks the first epoch each layer hits its minimum rank---every layer collapses within the first 100 epochs. (c) Singular-value spectra of the last two weight matrices across epochs $\{0,10,\dots,190\}$. The red shaded tail beyond $i>r$ is the energy that DLRT permanently discards: $98.8\%/98.6\%$ of the total $\|W\|_F^2$ for the two layers.}
  \label{fig_motivation}
\end{figure}

However, we observe that DLRT becomes unreliable precisely where compression matters most. As illustrated in Figure~\ref{fig_motivation}, on VGG-19/CIFAR-100 with an about $69.6\%$ initial compression, DLRT fails to find a trainable subnetwork: ranks contract rapidly in the early epochs and the truncated singular-value tail continues to carry the majority of each layer's energy throughout training. This echoes the expansion-limitation phenomenon reported by~\cite{frankle2020linear}, and here it takes the form of a self-reinforcing loop in which reduced expressive capacity drives ranks downward, and lower ranks further erode expressive capacity.

Fortunately, some scholars have pointed out that whether a trainable subnetwork can be found is related to the stability of the sub-network~\cite{frankle2019stabilizing}. We argue that the low-rank instability is not a mere optimization artifact but a direct consequence of the gradient flow based on dynamic low-rank approximation (DLRA)~\cite{koch2007dynamical}. Inspired by the analysis of dynamic orthogonal approximation in~\cite{koch2007dynamical}, we derive the gradient flow satisfied by the best rank-$r$ approximation of the full-parameter weights and compare it with DLRA. The comparison reveals two distinct sources of error: \emph{gradient locality}---DLRA evaluates the gradient at the low-rank neighbour rather than at the full-parameter point, which is a price we must pay for staying on the manifold; and \emph{curvature coupling}---DLRA entirely drops a correction term in which the discarded singular components $(\sigma_{r+j}, u_{r+j}, v_{r+j})$ interact with the retained ones $(\sigma_i,u_i,v_i)$ through a factor proportional to $\sigma_{r+j}/(\sigma_i^{2}-\sigma_{r+j}^{2})$. Under aggressive compression the spectral gap $\sigma_i^{2}-\sigma_{r+j}^{2}$ is small, so the coupling is large and its omission is the dominant source of the instability observed in Figure~\ref{fig_motivation}.

Guided by this analysis, we propose \textbf{Stabilized Dynamic Low-Rank Training (SDLRT)}. Rather than raising the rank, which is expensive, SDLRT keeps a lightweight \emph{compensation buffer} $(U_{\mathrm{neg}}, V_{\mathrm{neg}})$ that stores the top neglected singular directions from the previous iteration and reinjects them when updating the small factor $S$, partially recovering the curvature-coupling term missing from DLRT. A negative feedback mechanism on the singular-value tolerance $\vartheta$ further stabilizes the rank of each layer. Our main contributions are summarized as follows:
\begin{enumerate}[1.]
\item We derive the gradient flow of the best rank-$r$ approximation of the full-parameter weights in closed SVD form (Theorem~\ref{thm:best-lowrank}) and identify the curvature-coupling term whose omission makes DLRA unstable under high compression.
\item Built on this insight, we propose a stabilized dynamic low-rank training method, which augments the basis-extension step with the compensation buffer and stabilizes the adapted rank through a negative feedback on the truncation tolerance $\vartheta$. Additionally, we provide an exactness guarantee (Theorem~\ref{thm:exactness}) and an error bound (Theorem~\ref{thm:error}), confirming that the buffer does not alter the starting point and that convergence is preserved.
\item Extensive experiments across computer vision tasks and natural language processing show that SDLRT reliably identifies trainable low-rank subnetworks under aggressive compression, and significantly improves performance on natural language understanding.
\end{enumerate}

\section{Related Work}
\label{sec:related}
\textbf{Model compression.} Pruning is the dominant route to reducing the memory and compute footprint of deep networks. Unstructured pruning attains high sparsity but yields limited speedup without specialized kernels~\cite{han2016deep,frantar2023sparsegpt,sun2024wanda}, whereas structured pruning removes entire channels, filters, or attention heads at the cost of sharper accuracy degradation under aggressive compression~\cite{fang2023depgraph,ma2023llmpruner,ashkboos2024slicegpt,he2024structured}. A more fundamental bottleneck is the \emph{pretrain--prune--retrain} pipeline itself: its iterative cost scales poorly with model size, and the final retraining step often fails to recover accuracy once the compression ratio exceeds a task-dependent threshold~\cite{liu2019rethinking,renda2020comparing,cheng2024survey}. Low-rank factorization offers a complementary view by pruning redundant directions rather than individual entries~\cite{li2023losparse,zhang2023loraprune,idelbayev2020lowrank}. Yet applying post-SVD still requires a fully trained dense model and decouples the rank decision.

\textbf{Lottery ticket hypothesis.} Some scholars conjectured that dense networks contain sparse ``winning tickets'' which, trained in isolation, match the accuracy of the full model~\cite{frankle2019lottery}. Subsequent work stabilized the hypothesis at scale through weight and learning-rate rewinding~\cite{renda2020comparing,frankle2020linear}, established transferability across datasets and tasks~\cite{morcos2019one,chen2021lottery_cv,chen2020lottery_bert}, and proved strong-lottery variants in which subnetworks are uncovered in randomly initialized networks without any weight update~\cite{ramanujan2020whats,malach2020proving,pensia2020optimal,dacunha2022proving}. The hypothesis has since been extended to GNNs~\cite{chen2021unified,hui2023rethinking}, GANs and diffusion models~\cite{chen2021gans,jiang2023successfully}, and pretrained language models~\cite{chen2020lottery_bert,liu2024survey_lth}. Nevertheless, identifying a winning ticket remains the central difficulty: iterative magnitude pruning demands many train--prune cycles~\cite{paul2023unmasking,zhang2021validating}, and sanity checks reveal that the outcome is highly sensitive to learning rate, training length, and residual-connection topology~\cite{ma2021sanity,burkholz2022convolutional}. A parallel line pursues low-rank winning tickets in spectral space~\cite{schotthofer2022low}; however, as we show in Section~\ref{sec:derivation}, the standard dynamic low-rank procedure used to uncover them collapses precisely when compression matters most.

\textbf{Dynamic low-rank approximation.} Dynamic low-rank approximation (DLRA) restricts optimization to the manifold of rank-$r$ matrices via a Galerkin projection of the gradient flow~\cite{koch2007dynamical}, with robust numerical integrators developed in~\cite{ceruti2022unconventional,ceruti2022rankadaptive}. Building on these tools, DLRT~\cite{schotthofer2022low} and its rank-adaptive extension~\cite{zangrando2024rankadaptive} update only the SVD factors $(U, S, V)$ during training, eliminating both specialized initialization and post-hoc factorization. Despite this progress, a shared weakness persists: at every step these methods \emph{discard the truncated spectral tail}, implicitly treating small singular components as irrelevant. In Section~\ref{sec:derivation} we show that this omission drops a curvature-coupling term under aggressive compression and drives the rank collapse as shown in Figure~\ref{fig_motivation}. Our method directly targets this neglected interaction via a lightweight compensation buffer, at negligible additional cost.

\section{Method}

\subsection{Low-rank Training}

Since over-parameterized models tend to have low-rank property~\cite{arora2019implicit,galanti2024sgd}, low-rank methods are commonly used strategies to reduce memory storage and computational costs during the training and inference phases of deep learning models. As shown in Figure~\ref{fig_procedure}, low-rank training based on SVD~\cite{yang2020learning,zhang2023adalora} parameterizes the weights of each layer $W$ as $W=USV^{\top}$, where $U$ and $V$ are orthogonal matrices with $r$ columns, and $S$ is a small matrix of size $r\times r$. During model training or inference, the update of the weight matrix is transformed into updates of the low-rank factors $U, V, S$, which can significantly save resources. However, unconstrained alternating or direct training may cause large oscillations in the training loss and slower convergence~\cite{khodak2021init,razin2022implicit}. It is known that the maximum curvature (after normalizing the normal vector) of the point on the low-rank manifold corresponds to the inverse of the smallest singular value~\cite{feppon2018geometric}. Therefore, when the smallest singular value approaches zero, the curvature at that point diverges. To ensure the accuracy of the approximation, the largest singular value that is ignored must be very small, meaning that the remained smallest singular value should not be too large. To mitigate this phenomenon, most researchers address it from the aspect: the constraint of the training process~\cite{zhao2024galore,schotthofer2022low,savostianova2023robust}.

% Training based on dynamic orthogonal approximation (DOA) or dynamic low-rank approximation (DLRA) is a popular method for low-rank training. 
Given the loss function $\mathcal{L}$, we assume that the training phase is considered as a continuous process that changes over time, where the weights of the full-parameter network are denoted as $W_f(t)$, and their time-derivative is:
\begin{equation}
W_f(t) = W_f(0)+\int_0^t\dot W_f(s)\,ds,\quad\text{where}\quad
\dot W_f(s)=-\nabla_{W_f} \mathcal{L}(W_f(s)).
\label{equ:full_net}
\end{equation}
To constrain the full-parameter weights onto a low-rank manifold, DLRA projects the initial point onto a rank-$r$ manifold and subsequently projects the gradients onto the tangent plane of the manifold:
\begin{equation}
% \left\{
% \begin{array}{cl}
% \dot{W}(t) & =\Pi_{\mathcal{T}(W)}(\mathcal{L}(W(t))) \\
% W(0) & =\Pi_{\mathcal{M}_r}(W_f(0)),
% \end{array}\right.
\dot{W}(t) =\Pi_{\mathcal{T}(W)}(-\nabla_W\mathcal{L}(W(t)))\quad \text{and}\quad W(0) =\Pi_{\mathcal{M}_r}(W_f(0)),
\end{equation}
where $W(t)$ denotes the low-rank weights projected by DLRA, $\mathcal{M}_r$ denotes the manifold of rank-$r$ matrices, $\mathcal{T}(W)$ is the tangent space at $W\in \mathcal{M}_r$, $\Pi_{\mathcal{T}(W)}$ is the orthogonal projection onto the tangent space $\mathcal{T}(W)$, and $\Pi_{\mathcal{M}_r}$ is the orthogonal projection onto the manifold $\mathcal{M}_r$, i.e., the map that keeps the first $r$ singular values and sets the rest to zero.

Specifically, assuming the low-rank network weights are parameterized as $W(t)=USV^{\top}\in \mathcal{M}_r$, the optimization problem we consider is formulated as follows:
\begin{equation}
\begin{array}{cl}
\min & \| \dot W(t)+\nabla_W \mathcal{L}(W(t))\|_F \\
\text{s.t.} & W(t)=USV^{\top}\in \mathcal{M}_r \\
& \dot W(t)=\dot U S V^{\top}+U\dot S V^{\top}+US\dot V^{\top}\in \mathcal{T}(W(t)).
\end{array}
\end{equation}
Then, using the Gauge conditions $U^{\top}\dot U=0$ and $V^{\top}\dot V=0$ as presented in~\cite{koch2007dynamical}, we obtain the gradient flow equation for the low-rank factors:
\begin{equation}
\begin{cases}
\dot{S}=-U^{\top}\nabla_{W}\mathcal{L}(W(t))V, \\
\dot{U}=-P^{\perp}_U\nabla_{W}\mathcal{L}(W(t))VS^{-1}, \\
\dot{V}=-P^{\perp}_V\nabla_{W}\mathcal{L}(W(t))^{\top}US^{-\top},
\end{cases}
\label{equ:dlra}
\end{equation}
where $P^{\perp}_U=(I-UU^{\top})$ and $P^{\perp}_V=(I-VV^{\top})$ are used to project the original matrix onto the spaces orthogonal to the span of $U$ and $V$, respectively.

However, we observe that at higher compression rates, the low-rank networks trained based on Equation~\eqref{equ:dlra} are highly unstable as shown in Figure~\ref{fig_motivation}, with potentially low expressive power, making it difficult to identify trainable ``winning tickets'' in the network. In Figure~\ref{fig_motivation}, we plot the performance of VGG-19 trained with the DLRT method compared to the baseline under an initial compression rate of $69.6\%$, as well as the rank variations across multiple layers within the network. The instability of the network leads to unstable ranks of the network weights trained based on Equation~\eqref{equ:dlra}. This rank instability, in turn, negatively impacts the expressive power of the network, further exacerbating its instability and creating a vicious cycle.

In fact, the weights obtained through training with Equation~\eqref{equ:dlra} still exhibit certain offset compared to the best low-rank approximation of the original weights derived from full-parameter gradient training.

\subsection{Gradient Flow of the Best Low-Rank Approximation}
\label{sec:derivation}

To directly characterize the gap between DLRA and the best low-rank trajectory, we now derive the gradient flow equations satisfied by the best rank-$r$ approximation of the full-parameter weight $W_f(t)$. The derivation makes explicit which information is discarded by DLRA and motivates the compensation mechanism introduced in Section~\ref{sec:method}.

Throughout this subsection, we write $\mathcal{G} \triangleq \nabla_W \mathcal{L}(W_f(t))$ for brevity and recall that $\dot{W}_f = -\mathcal{G}$.

\textbf{Setup and notation.} Let the SVD of the full-rank weight matrix be $W_f = \sum_{i=1}^{r+k} \sigma_i u_i v_i^\top$, partitioned as
\begin{equation}
W_f = U\,\Sigma_r V^\top + U_\perp\Sigma_\perp V_\perp^\top,
\end{equation}
where $U=[u_1,\dots,u_r]$, $U_\perp=[u_{r+1},\dots,u_{r+k}]$, $V=[v_1,\dots,v_r]$, $V_\perp=[v_{r+1},\dots,v_{r+k}]$, $\Sigma_r=\mathrm{diag}(\sigma_1,\dots,\sigma_r)$, and $\Sigma_\perp=\mathrm{diag}(\sigma_{r+1},\dots,\sigma_{r+k})$. The best rank-$r$ approximation is $USV^\top$ with $S=\Sigma_r$. We adopt the gauge conditions $U^\top \dot U = 0$ and $V^\top \dot V = 0$, which guarantee that $U^\top U = I_r$ and $V^\top V = I_r$ are preserved along the flow. The following orthogonality relations will be used repeatedly:
\begin{equation}
W_f V = U\Sigma_r,\quad W_f^\top U = V\Sigma_r,\quad V_\perp^\top V = 0,\quad U_\perp^\top U = 0.
\label{eq:ortho}
\end{equation}

\textbf{Deriving the coupled system.} The optimality of the rank-$r$ truncation requires $(I-UU^\top)W_f V = 0$ and $(I-VV^\top)W_f^\top U = 0$ at all times. Differentiating the first constraint and using \eqref{eq:ortho}:
\begin{equation}
-\dot U \underbrace{U^\top W_f V}_{=\Sigma_r}
-U\underbrace{\dot U^\top W_f V}_{=\dot U^\top U\Sigma_r=0} + P_U^\perp \dot W_f V + P_U^\perp W_f \dot V = 0,
\end{equation}
where $P_U^\perp = I - UU^\top$. Since $P_U^\perp W_f (I - VV^\top) = U_\perp \Sigma_\perp V_\perp^\top$ and $\dot V = (I - VV^\top)\dot V$ by the gauge conditions, we obtain
\begin{equation}
\dot U\Sigma_r = -P_U^\perp\mathcal{G}V + U_\perp\Sigma_\perp V_\perp^\top\dot V.
\label{eq:star}
\end{equation}
By symmetry, differentiating $(I-VV^\top)W_f^\top U = 0$ yields
\begin{equation}
\dot V\Sigma_r = -P_V^\perp\mathcal{G}^\top U + V_\perp\Sigma_\perp U_\perp^\top\dot U.
\label{eq:starstar}
\end{equation}
Equations~\eqref{eq:star} and~\eqref{eq:starstar} are coupled through the off-diagonal block $\Sigma_\perp$, which encodes the very interaction between retained and discarded components that DLRA neglects.

\textbf{Expansion in the singular-vector basis.} By the gauge conditions, $\dot U \in \mathrm{col}(U)^\perp$ and $\dot V \in \mathrm{col}(V)^\perp$. We may therefore expand $\dot U = U_\perp\,\mathcal{C}$ and $\dot V = V_\perp\,\mathcal{D}$, with coefficient matrices $\mathcal{C},\mathcal{D}\in\mathbb{R}^{k\times r}$. Substituting into~\eqref{eq:star} and multiplying to the left by $U_\perp^\top$ gives
\begin{equation}
\mathcal{C}\Sigma_r = -U_\perp^\top\mathcal{G}V + \Sigma_\perp\mathcal{D},
\label{eq:C}
\end{equation}
while~\eqref{eq:starstar} analogously yields
\begin{equation}
\mathcal{D}\Sigma_r = -V_\perp^\top\mathcal{G}^\top U + \Sigma_\perp\mathcal{C}.
\label{eq:D}
\end{equation}
For each pair $(i,j)$ with $1\le i \le r$, $1\le j \le k$, define the shorthand
\begin{equation}
a_{ji} \triangleq u_{r+j}^\top \mathcal{G}v_i,\qquad b_{ji} \triangleq v_{r+j}^\top \mathcal{G}^\top u_i = u_i^\top \mathcal{G} v_{r+j}.
\end{equation}
Equations~\eqref{eq:C} and~\eqref{eq:D} read component-wise:
\begin{equation}
\sigma_i\mathcal{C}_{ji} = -a_{ji} + \sigma_{r+j}\mathcal{D}_{ji}, \qquad
\sigma_i\mathcal{D}_{ji} = -b_{ji} + \sigma_{r+j}\mathcal{C}_{ji}.
\end{equation}
Eliminating $\mathcal{D}_{ji}$ and using the distinct-singular-value assumption $\sigma_i^2 \ne \sigma_{r+j}^2$:
\begin{equation}
\begin{cases}
\mathcal{C}_{ji} = -\dfrac{\sigma_i a_{ji} + \sigma_{r+j} b_{ji}}{\sigma_i^2 - \sigma_{r+j}^2},\\[6pt]
\mathcal{D}_{ji} = -\dfrac{\sigma_{r+j}a_{ji} + \sigma_i b_{ji}}{\sigma_i^2 - \sigma_{r+j}^2}.
\end{cases}
\label{eq:Dji}
\end{equation}

\textbf{Reassembling the evolution equations.} Substituting~\eqref{eq:Dji} back into $\dot U = U_\perp \mathcal{C}$ and expressing the result as a rank-one sum:
\begin{equation}
\dot U = -P_U^\perp\mathcal{G}VS^{-1} - \left[\sum_{\substack{1\le i \le r\\ 1\le j \le k}}\frac{\sigma_{r+j}}{\sigma_i^2 - \sigma_{r+j}^2}(\sigma_{r+j}u_{r+j}^\top \mathcal{G}v_i + \sigma_i u_i^\top \mathcal{G}v_{r+j})u_{r+j}v_i^\top\right]VS^{-1},
\end{equation}
where the first term arises from extracting the $\sigma_i^{-1}a_{ji}$ part of $\mathcal{C}_{ji}$, and the bracketed sum collects the remaining $\sigma_{r+j}$-dependent corrections. To verify the splitting, observe
\begin{equation}
\mathcal{C}_{ji} = \underbrace{-\frac{a_{ji}}{\sigma_i}}_{\text{DLRA part}} - \underbrace{\frac{\sigma_{r+j}(\sigma_{r+j}a_{ji} + \sigma_i b_{ji})}{\sigma_i(\sigma_i^2 - \sigma_{r+j}^2)}}_{\text{curvature correction}},
\end{equation}
which follows from the identity $\frac{\sigma_i}{\sigma_i^2 - \sigma_{r+j}^2} = \frac{1}{\sigma_i} + \frac{\sigma_{r+j}^2}{\sigma_i(\sigma_i^2 - \sigma_{r+j}^2)}$. The DLRA part equals $-U_\perp U_\perp^\top \mathcal{G}\,V \Sigma_r^{-1} = -P_U^\perp \mathcal{G}\,V S^{-1}$ since $P_U^\perp = U_\perp U_\perp^\top$ on $\mathrm{col}(U)^\perp$.

The equation for $\dot V$ follows by applying the same argument to~\eqref{eq:Dji} with the symmetrical exchange of the roles of $(U,V)$ and $(\sigma_i,\sigma_{r+j})$.\\
% \begin{equation}
% \dot V = -P_V^\perp\mathcal{G}^{\top}US^{-T} - \left[\sum_{\substack{1\le i \le r\\ 1\le j \le k}}\frac{\sigma_{r+j}}{\sigma_i^2 - \sigma_{r+j}^2}(\sigma_{r+j}v_{r+j}^\top \mathcal{G}^{\top}u_{r+j} + \sigma_i v_i^\top \mathcal{G}^{\top}u_i)v_{r+j}u_i^\top\right]US^{-T}.
% \end{equation}
Finally, the evolution of $S$ is straightforward:
\begin{equation}
\dot S = \frac{d}{dt}(U^\top W_f V) = U^\top \dot W_f V = -U^\top \mathcal{G}V.
\end{equation}
Collecting these three results yields exactly the coupled system stated in Theorem~\ref{thm:best-lowrank} below.

\begin{theorem}
\label{thm:best-lowrank}
Let $W_f(t)$ denote the original network weights obtained through full-parameter gradient training, as defined in Equation~\eqref{equ:full_net}. Assume that there is no repeated singular value crossing the truncation, and let $USV^\top$ represent the best rank-$r$ approximation of $W_f(t)$. Then its derivative is given by
\begin{equation}
\begin{cases}
\dot U = -P_U^\perp\mathcal{G}VS^{-1}
- \left[\displaystyle\sum_{\substack{1\le i \le r\\ 1\le j \le k}}(\mathcal{P}_{ij}u_{r+j}^\top \mathcal{G}v_i + \mathcal{Q}_{ij} u_i^\top\mathcal{G}v_{r+j})u_{r+j}v_i^\top\right]VS^{-1},\\[14pt]
\dot V = -P_V^\perp\mathcal{G}^{\top}US^{-T}
- \left[\displaystyle\sum_{\substack{1\le i \le r\\ 1\le j \le k}}(\mathcal{P}_{ij}v_{r+j}^\top \mathcal{G}^{\top}u_{i} + \mathcal{Q}_{ij} v_i^\top\mathcal{G}^{\top}u_{r+j})v_{r+j}u_i^\top\right]US^{-T},\\[14pt]
\dot S = -U^\top \mathcal{G}V,
\end{cases}
\end{equation}
where $\mathcal{G} = \nabla_W \mathcal{L}(W_f(t))$, $\mathcal{P}_{ij}=\sigma_{r+j}^2/(\sigma_i^2-\sigma_{r+j}^2)$, and $\mathcal{Q}_{ij}=\sigma_i\sigma_{r+j}/(\sigma_i^2-\sigma_{r+j}^2)$.
\end{theorem}

\textbf{Interpretation.} Comparing Theorem~\ref{thm:best-lowrank} with the DLRA update in Equation~\eqref{equ:dlra}, two discrepancies stand out:

\textit{(1) Gradient locality.} The best rank-$r$ flow evaluates the gradient $\nabla_W \mathcal{L}(W_f(t))$ at the full-parameter point $W_f(t)$, whereas DLRA evaluates the gradient $\nabla_W \mathcal{L}(W(t))$ at the neighboring low-rank point $W(t)$. This mismatch is a necessary price of working on the low-rank manifold and cannot be eliminated without sacrificing the memory savings.

\textit{(2) Curvature coupling.} The bracketed correction terms, governed by the factor $\sigma_{r+j}/(\sigma_i^2 - \sigma_{r+j}^2)$, explicitly couple the retained singular triplets $(\sigma_i, u_i, v_i)$ with the truncated ones $(\sigma_{r+j}, u_{r+j}, v_{r+j})$. When the gap $\sigma_i^2 - \sigma_{r+j}^2$ is small---precisely the regime encountered under aggressive compression---these corrections become non-negligible. DLRA drops them entirely, which we identify as the root cause of the rank-collapse and accuracy-degradation phenomena documented in Figure~\ref{fig_motivation}.

This analysis directly motivates our design: rather than discarding the truncated basis $(U_\perp, V_\perp)$ at each step, SDLRT \emph{preserves it as a compensation buffer $(U_{\mathrm{neg}}, V_{\mathrm{neg}})$ and reinjects it into the subsequent update of $S$}, thereby partially recovering the curvature coupling missing from Equation~\eqref{equ:dlra}.

\begin{wrapfigure}{r}{0.5\textwidth}
  \centering
  \includegraphics[width=0.98\linewidth]{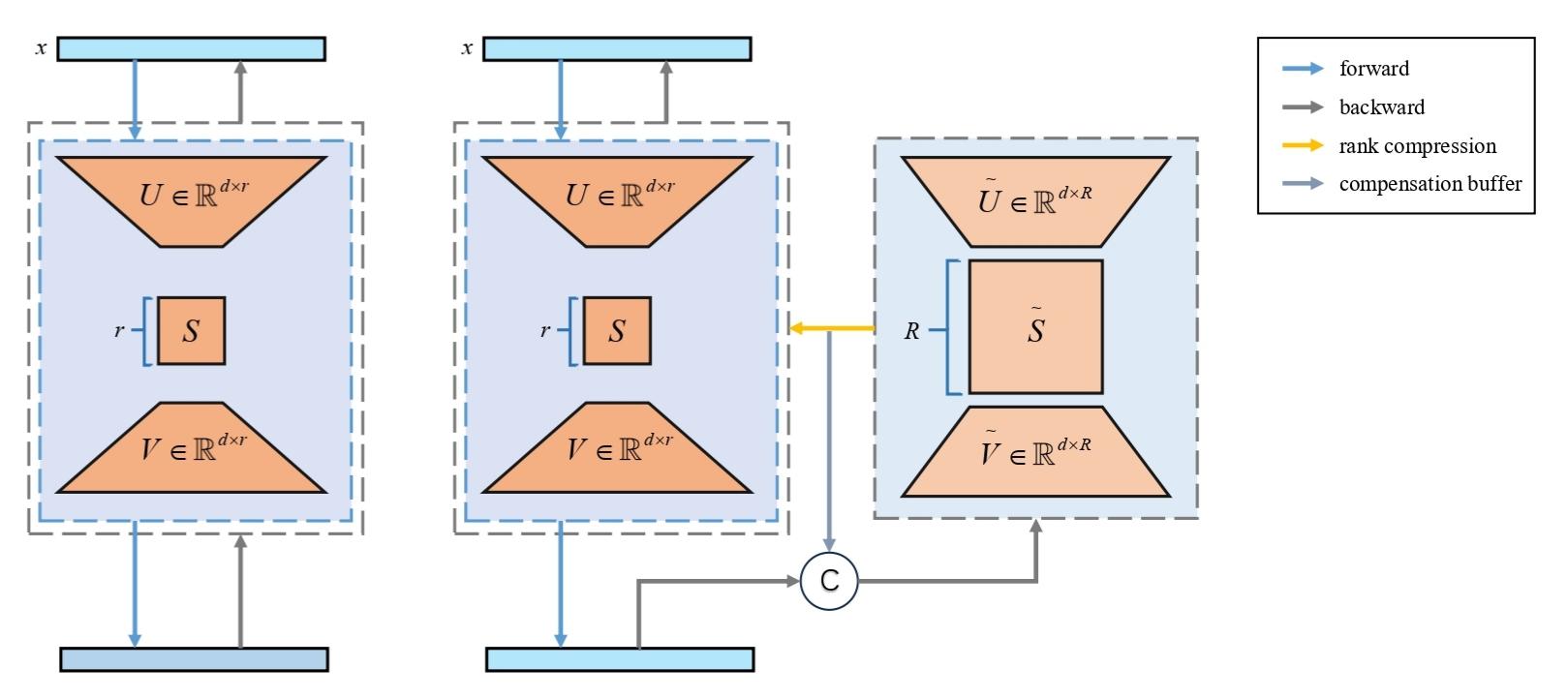}
  \caption{The procedures of SVD-based low-rank training (left) and the stable SVD-based low-rank training with compensation (right).}
  \label{fig_procedure}
\end{wrapfigure}

\subsection{Stabilizing the Dynamic Low-Rank Training}
\label{sec:method}
To better approximate the best rank-$r$ approximation, we retain the discarded singular value components at each update to assist the next one, as shown in the right panel of Figure~\ref{fig_procedure}. Instead of increasing the rank $r$, which is computationally expensive, we only incorporate small singular value information when updating the small matrix $S$ with the gradient flow, leaving the dimensions of the basis matrices unchanged. In this way, the additional computational and storage costs are basically only related to the size of the small matrix $S$. Since $r \ll \min(m, n)$, the additional costs introduced by this framework are essentially negligible. Based on the framework illustrated in the right panel of Figure~\ref{fig_procedure}, we propose a more stable variant of DLRT with the following update strategy for each layer's low-rank factors:

\begin{enumerate}[1.]
\item Initialize $K_0=US$, $L_0=VS^{\top}$ and update $K$ and $L$ separately via gradient flow to obtain $K_1$ and $L_1$.
\item Extend $(K_1, U)$ and $(L_1, V)$ with the truncated basis matrices $U_{\mathrm{neg}}$ and $V_{\mathrm{neg}}$ from the previous iteration to form the compensated matrices $(K_1, U, U_{\mathrm{neg}})$ and $(L_1, V, V_{\mathrm{neg}})$.
\item Project the compensated matrices onto the Stiefel manifold (e.g., via QR decomposition) to obtain the orthogonal bases $U_1$ and $V_1$.
\item Initialize $S_0=U_1^{\top}W_{\text{prev}}V_1$, where $W_{\text{prev}}=USV^{\top}$. And in the newly updated space of matrices $U_1 S V_1^{\top}$, perform a Galerkin-based update to obtain a new small matrix $S_1$.
\item Compute the SVD of $S_1$, i.e., $S_1=\hat U\hat S\hat V^{\top}$, and mark the rank of $S_1$ as $r_{\max}$. Determine the new rank $r_1$ based on the tolerance $\vartheta=\tau\|\hat S\|$ and update the matrices:
\begin{equation}
\begin{cases}
U_1=U_1\hat U[:, :r_1],\ S_1=\hat S[:r_1,:r_1],\ V_1=V_1\hat V[:, :r_1],\\
U_{\mathrm{neg}}=U_1\hat U[:, r_1:\min(2r_1,r_{\max})],\ V_{\mathrm{neg}}=V_1\hat V[:, r_1:\min(2r_1,r_{\max})].
\end{cases}
\nonumber
\end{equation}
\item Negative feedback mechanism: If the new rank $r_1$ is less than the initial rank $r_0$, multiply the tolerance factor $\tau$ by $\omega$, where $\omega < 1$ is a feedback factor.
\end{enumerate}

\textbf{Computational complexity. } For the training of each layer, the first step---initializing $K_0$ and $L_0$---requires $O(nr^2)$ and $O(mr^2)$ operations, respectively, while the minimal computational cost of calculating the gradient flow with respect to $K$ and $L$ is $O(r(n + m))$. The second step incurs no computational cost. The third step (QR decomposition on the compensation matrix) requires $O(9nr^2)$ and $O(9mr^2)$ operations. The fourth step (initializing $S_0$) requires $O(3nr^2 + 3mr^2 + 6r^3)$ operations, with the minimal gradient update costing $O(3r(n + m) + 9r^2)$. The SVD computation in the worst case incurs a cost of $O(r^3)$, and computing the updated basis matrices for the new rank requires $O(4r^2(n + m))$. The tolerance adjustment in the sixth step incurs no computational cost. In summary, given that $r \ll \min(n, m)$, the computational cost for each layer can be controlled within $O(r^2(n + m))$, comparable to naive DLRT and significantly lower than the full-rank baseline's $O(nm)$.

\subsection{Exactness Property and Error Analysis}
In this section, we conduct a theoretical analysis of the proposed method. First, to demonstrate the exactness of our low-rank training, we assume the existence of a well-performing network with low-rank parameters under a certain initialization, and prove that our method can accurately learn these low-rank parameters.

\begin{theorem}[Exactness]
\label{thm:exactness}
Assume that $W(t)$ represents the low-rank network weights obtained at time $t$ through gradient-based training, starting from initialized weights $W(0)$. Under SDLRT, if the network is initialized as $W(0)$ and its parameters are factorized as $U(0), S(0), V(0)$, then after a certain training period of $T$, the following relationship holds:
$U(t)S(t)V(t)^{\top}=W(t)$.
\end{theorem}

However, in the absence of additional constraints, networks trained via gradient-based methods may not exhibit significant low-rank properties. In practice, if the loss function is $L$-Lipschitz continuous and $B$-bounded, we can derive the following:

\begin{theorem}[Error Bound]\label{thm:error}
Assume that $W_f(t)$ represents the full-rank network weights obtained at time~$t$ through
training based on Equation~\ref{equ:full_net}, starting from the initialized weights
$W_f(0)$. Moreover, if $W_f(0)$ is projected onto a low-rank manifold as initialization,
represented by $U(0),S(0),V(0)$, then after a training period of $T=n\eta$ where $\eta$ denotes the time of each step, the SDLRT
iterates $\hat{U}_t,\hat{S}_t,\hat{V}_t$ satisfy
\begin{equation}\label{eq:bound_sdlrt}
  % \lVert \hat{U}_{t}\hat{S}_{t}\hat{V}_{t}^{\top}-W(t)\rVert_{F}
  % \le c_{0}\delta + c_{1}\varepsilon + c_{2}h
  %   + \gamma c_{3}n\theta,
  \|\hat U_t\hat S_t\hat V_t^{\top}-W_f(n\eta)\|_F\le
c_0\delta+c_1\gamma\varepsilon+c_2\eta+c_3\vartheta/\eta,
\nonumber
\end{equation}
where $c_0,c_1,c_2,c_3$ depend only on $L,B,T$, $\delta=\lVert U_0 S_0 V_0^{\top}-W_f(0)\rVert_F$ is the initial projection error, and $\gamma\in(0,1)$ is the spectral contraction factor induced by the compensation buffer. In particular, the improvement of upper bound over DLRT satisfies
\begin{equation}\label{eq:gap}
  E_{\mathrm{DLRT}}-E_{\mathrm{SDLRT}}=(1-\gamma)c_1\varepsilon>0.
  \nonumber
\end{equation}
\end{theorem}
% \begin{theorem}
% \label{thm:error}
% Assume that $W(t)$ represents the full-rank network weights obtained at time $t$ through training based on Equation~\eqref{equ:full_net}, starting from the initialized weights $W(0)$. Moreover, if $W(0)$ is projected onto a low-rank manifold as initialization, represented by $U(0), S(0), V(0)$, then after a training period of $t = nh$, the following relationship holds:
% \begin{equation}
% \| U_tS_tV_t^{\top}-W(t)\|_F \le c_0\delta + c_1\varepsilon + c_2 h + c_3 n\theta,
% \end{equation}
% where $\delta = \|U_0S_0V_0^{\top}-W(0)\|_F$ is the initial projection error. %$\varepsilon$ is the truncation error per step, $h$ is the step size, $\theta$ is the accumulated coupling correction bound, and $c_0, c_1, c_2, c_3$ are constants depending only on $L$, $B$ and the spectral gap.
% \end{theorem}

% \begin{figure}[t]
%   \centering
%   \includegraphics[width=\linewidth]{fig/fig_alexnet_vgg.jpg}
%   \caption{Comparison of accuracy under different training and testing compression rates on CIFAR-10. Error bars represent the minimum and maximum across 5 samples. }
%   \label{fig_cifar10}
% \end{figure}

% 换为1*4排列
\begin{figure}[t]
  \centering
  \begin{subfigure}[b]{0.49\linewidth}
    \centering
    \includegraphics[width=\linewidth]{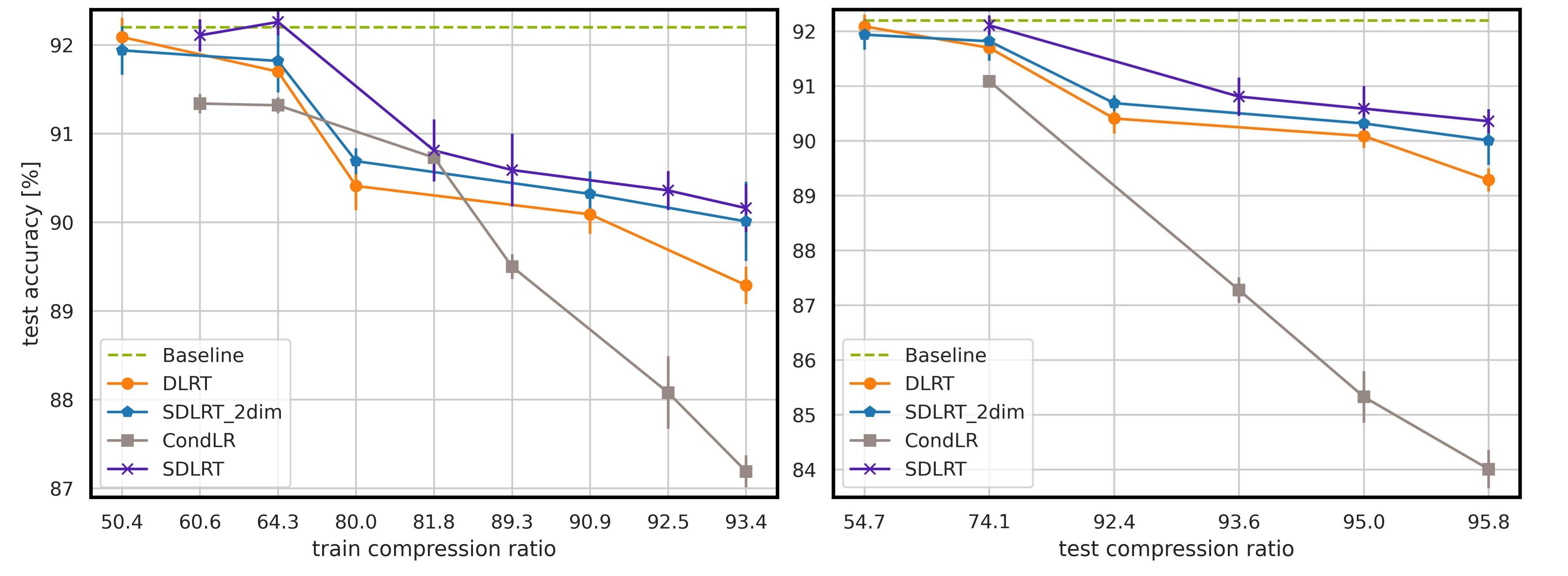}
    \caption{VGG-16 on CIFAR-10}
    \label{fig_cifar10_vgg}
  \end{subfigure}
  \hfill
  \begin{subfigure}[b]{0.49\linewidth}
    \centering
    \includegraphics[width=\linewidth]{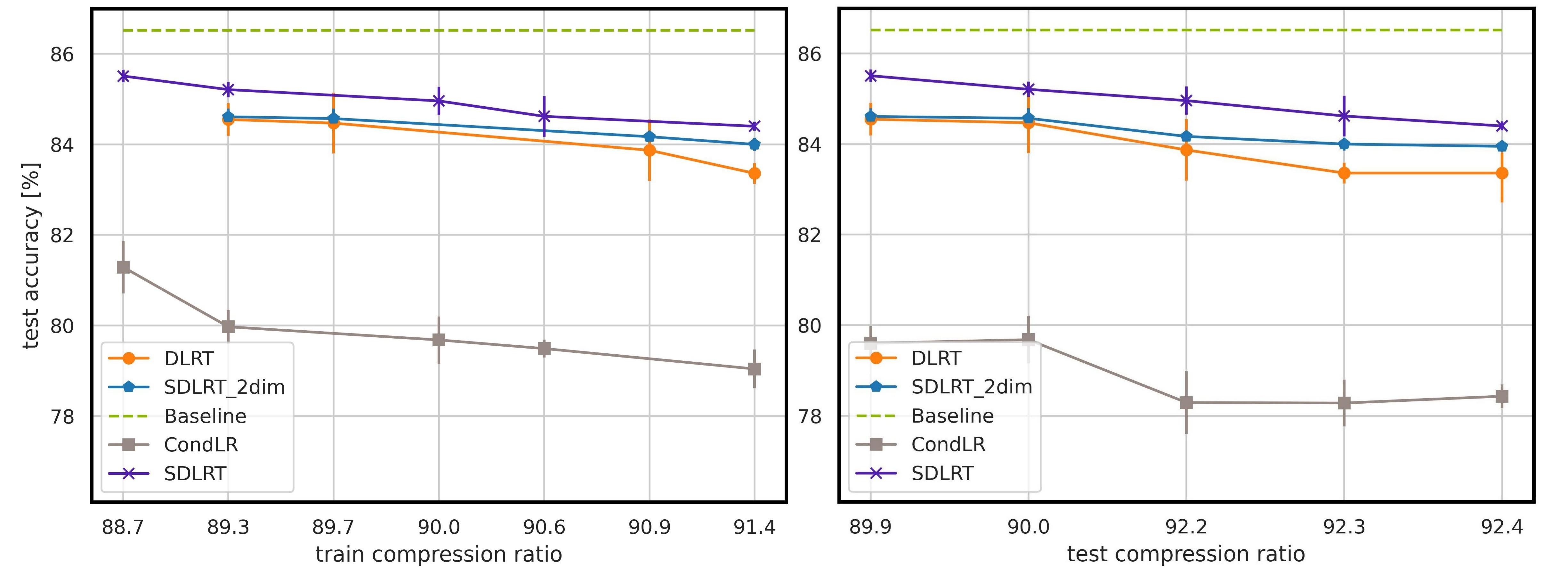}
    \caption{AlexNet on CIFAR-10}
    \label{fig_cifar10_alex}
  \end{subfigure}
  \caption{Comparison of accuracy under different training and testing compression on CIFAR-10. SDLRT and $\text{SDLRT}_{\text{2dim}}$ achieve the top two highest scores. Error bars represent the minimum and maximum across 5 samples.}
  \label{fig_cifar10}
\end{figure}

\begin{figure}[t]
    \centering
    \begin{subfigure}[b]{0.49\linewidth}
        \centering
        \includegraphics[width=\linewidth]{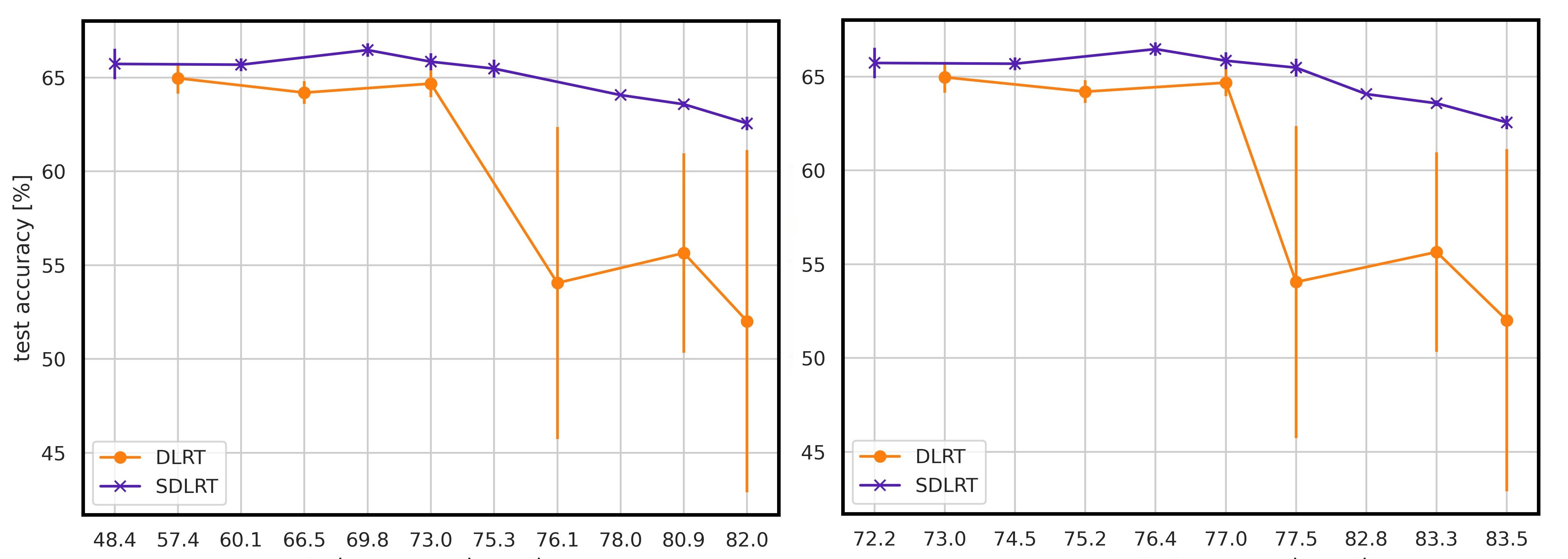}
        \caption{VGG-19 on CIFAR-100}
        \label{fig_cifar100}
    \end{subfigure}
    \hfill
    \begin{subfigure}[b]{0.49\linewidth}
        \centering
        \includegraphics[width=\linewidth]{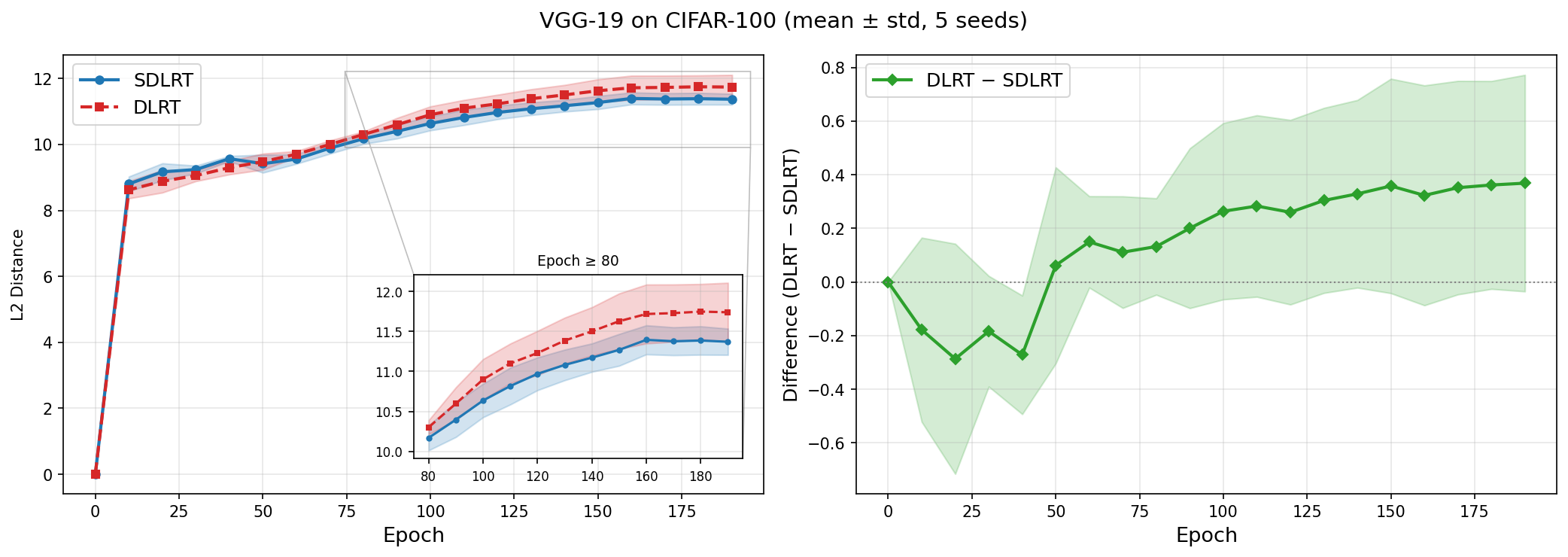}
        \caption{L2 Distance of VGG-19 on CIFAR-100}
        \label{fig_distance}
    \end{subfigure}
    \caption{On VGG-19/CIFAR-100, SDLRT can stably identify winning tickets under extreme compression rates, whereas DLRT cannot. Error bars represent the minimum and maximum across 5 samples.}
    \label{fig_cifar100_all}
\end{figure}

% \begin{figure}[t]
%   \centering
%   \includegraphics[width=\linewidth]{fig/fig_vgg_cifar100.jpg}
%   \caption{For complex tasks, SDLRT can stably identify winning tickets under extreme compression rates, whereas DLRT cannot. Error bars represent the minimum and maximum across 5 samples.}
%   \label{fig_cifar100}
% \end{figure}

% \begin{figure}[t]
%   \centering
%   \includegraphics[width=\linewidth]{svd_distance_errbar_CIFAR100.png}
%   \caption{Distance between the low-rank network weights and the baseline weights on VGG-19/CIFAR-100. Error bars represent the minimum and maximum across 5 samples.}
%   \label{fig_distance}
% \end{figure}

\section{Experiments}
\label{experiments}
We evaluate SDLRT from three complementary perspectives. Section~\ref{sec:exp_cv} examines its compression capability on image classification, showing that SDLRT stably identifies high-performing low-rank subnetworks even under extreme compression where existing methods collapse. Section~\ref{sec:exp_stability} investigates training stability by tracking the weight-space distance between low-rank subnetworks and their full-rank counterparts throughout optimization, revealing that the proposed compensation mechanism yields consistently closer and less variable trajectories. Section~\ref{sec:exp_peft} demonstrates the broad applicability of SDLRT by deploying it as a low-rank adapter for parameter-efficient fine-tuning of pretrained language models on SuperGLUE. The experiments are performed on an Nvidia RTX3080 and Nvidia RTX4090D 24GB. The code is available in the supplementary material.

\textbf{Notation. } A low-rank subnetwork is a tuple $(W, r)$, where $W$ represents the weights and $r$ is the specified rank, as defined before. We write $\mathrm{SVD}(W, r)$ for the best rank-$r$ approximation of $W$. Suppose that, starting from the same initialization $W_0$: (i) full-rank training produces $W_t$ at iteration $t$, and (ii) low-rank training initialized with $\mathrm{SVD}(W_0, r_0)$ produces factors $U_t, S_t, V_t$. The distance at iteration $t$ is $d_t = \|\mathrm{SVD}(W_t, r_t)-U_tS_tV_t^{\top}\|_F$, and the stability of a low-rank training method is measured by $d_T$ at convergence.

\subsection{Image Classification under Aggressive Compression}
\label{sec:exp_cv}
\textbf{Setup.} We evaluate on CIFAR-10 (10 classes) and CIFAR-100 (100 classes), each with 50k training and 10k test images. We employ VGG-16 and AlexNet for CIFAR-10, and VGG-19 for CIFAR-100. We compare SDLRT against three methods: (i) the Baseline network trained with standard SGD at full rank; (ii) DLRT; (iii) CondLR~\cite{savostianova2023robust}, a low-rank training method that imposes approximate orthogonality constraints. To isolate the source of improvement, we additionally report $\mathrm{SDLRT}_{\mathrm{2dim}}$, a variant that applies the compensation directly to the original factors $K$ and $L$ using $(K, U_{\mathrm{neg}})$ and $(L, V_{\mathrm{neg}})$. Each configuration is evaluated over 5 independent runs with different random seeds. We train for 120 epochs with a mini-batch size of 128, initial learning rate 0.05 decayed by a factor of 0.1 at epochs 60 and 100, with momentum 0.1, and set the $\tau\in\{0.1,0.2,0.25,0.3,0.4\}$, $\omega=0.8$.

\textbf{Results.} Since accuracy is largely preserved at low compression rates, we focus the analysis on the high-compression regime where performance differences become meaningful.
\textbf{SDLRT discovers higher-reward winning tickets.} Figure~\ref{fig_cifar10} shows test accuracy on CIFAR-10 as the compression rate increases. At moderate compression all methods perform comparably, and SDLRT's compensation overhead can even introduce mild overfitting. However, once the compression rate exceeds $\approx 80\%$, a clear separation emerges: SDLRT maintains accuracy above $90\%$ on VGG-16 and above $84\%$ on AlexNet, while DLRT and CondLR degrade substantially. The comparison between $\mathrm{SDLRT}_{\mathrm{2dim}}$ and the DLRT group under the same dimensional settings directly demonstrates the superiority of the proposed algorithm.
\textbf{SDLRT stably identifies winning tickets under extreme compression.} Figure~\ref{fig_cifar100} compares VGG-19 on CIFAR-100 at progressively higher compression rates. When compression exceeds $76.1\%$, DLRT may fail to identify trainable low-rank subnetworks: accuracy collapses and variance across seeds explodes. SDLRT, in contrast, continues to discover stable winning tickets by supplementing discarded singular value information during factor updates, demonstrating the effectiveness of the compensation mechanism.

\subsection{Training Stability of Low-Rank Subnetworks}
\label{sec:exp_stability}
Beyond final accuracy, a practically important question is \textit{how closely does a low-rank training trajectory track the full-rank solution throughout optimization?} A smaller and more consistent distance $d_t$ implies that the low-rank method faithfully approximates the expressive capacity of the full network, reducing the risk of divergent training dynamics.

\textbf{Setup.} Using the same model architectures and datasets as Section~\ref{sec:exp_cv} (VGG-19 on CIFAR-100), we set $\tau=0.4$ and record the distance $d_t$ between the low-rank subnetwork and the full-rank baseline at every epoch for both DLRT and SDLRT. We report the mean and standard deviation across 5 seeds.

\textbf{Results.} Figure~\ref{fig_distance} plots $d_t$ over the course of training. \textbf{SDLRT converges to a closer neighborhood of the full-rank solution.} At convergence, the distance $d_T$ of SDLRT is markedly smaller than that of DLRT. On VGG-19 at $\tau=0.4$, Figure~\ref{fig_distance} shows that the terminal distance is reduced by approximately $10\%$ relative to DLRT.

\begin{table}[t]
\centering
\caption{Results (\%) on fine-tuning DeBERTa-base with SuperGLUE datasets. }
\label{tab:superglue}
\resizebox{\textwidth}{!}{%
\begin{tabular}{lccccccc}
\toprule
Method & CB & COPA & WSC & RTE & WiC & BoolQ & Avg. \\
\midrule
SDLRT (ours) & $83.63_{\pm1.34}$ & ${91.00}_{\pm1.67}$ & $\mathbf{91.51}_{\pm1.13}$ & $86.52_{\pm0.59}$ & $\mathbf{73.93}_{\pm0.88}$ & $84.29_{\pm0.20}$ & $\mathbf{85.15}$ \\
DLRT & $81.85_{\pm1.75}$ & $\mathbf{91.17}_{\pm0.98}$ & $90.06_{\pm1.32}$ & $86.28_{\pm1.07}$ & $73.41_{\pm0.61}$ & $\textbf{84.33}_{\pm 0.54}$ & $84.52$ \\
LoRA & $\mathbf{85.72}_{\pm2.26}$ & $90.50_{\pm0.55}$ & $84.93_{\pm10.53}$ & $\mathbf{87.30}_{\pm0.93}$ & $73.07_{\pm0.81}$ & $84.14_{\pm0.21}$ & $84.28$ \\
AdaLoRA & $82.74_{\pm2.16}$ & $90.67_{\pm2.07}$ & $63.46_{\pm0.00}$ & $85.92_{\pm0.82}$ & $73.49_{\pm0.54}$ & $83.92_{\pm0.14}$ & $80.03$ \\
LoRA+ & $85.12_{\pm2.69}$ & $89.83_{\pm1.83}$ & $89.58_{\pm2.23}$ & $87.24_{\pm1.16}$ & $73.27_{\pm0.57}$ & $83.80_{\pm0.56}$ & $84.81$ \\
\bottomrule
\end{tabular}}
\end{table}

% \begin{figure}[t]
%   \centering
%   \includegraphics[width=\linewidth]{combined_boxplots_2x3.png}
%   \caption{Distribution of the validation accuracy on DeBERTa-v3-base across six SuperGLUE tasks over six random seeds. The diamond marker indicates the mean and the circular marker indicates outliers. The box shows the interquartile range ($25\%$--$75\%$ quantiles) and the median.}
%   \label{fig_superglue}
% \end{figure}
\subsection{Parameter-Efficient Fine-Tuning on SuperGLUE}
\label{sec:exp_peft}
To demonstrate that SDLRT extends naturally beyond from-scratch training, we apply it as a low-rank adapter for parameter-efficient fine-tuning (PEFT) of pretrained language models.

\textbf{Setup.} We fine-tune DeBERTa-v3-base~\cite{he2023deberta} on six SuperGLUE~\cite{wang2019superglue} tasks: BoolQ, CB, COPA, RTE, WiC, and WSC. These tasks span a broad spectrum of NLU capabilities, from surface-level matching to nuanced commonsense reasoning, providing a comprehensive testbed for adapter quality. We compare with three widely adopted low-rank PEFT methods: LoRA~\cite{hu2022lora}, which injects fixed-rank decomposition matrices into frozen pretrained weights; LoRA+~\cite{hayou2024loraplus}, which assigns different learning rates to the two low-rank factors; and AdaLoRA~\cite{zhang2023adalora}, which adaptively allocates rank budgets across layers via importance scoring. For all methods, low-rank adaptation is applied to the query, key and value projection matrices of every self-attention layer. LoRA, LoRA+ and AdaLoRA use an initial rank $r=10$, resulting in $\approx 1.145$M trainable parameters. DLRT and SDLRT start at $r=10$ with singular-value tolerance $\tau = 0.02$, yielding $\approx1.159$M and $\approx 1.177$M ($\approx2.8\%$ overhead as the upper bound) trainable parameters respectively. All experiments are repeated over 6 random seeds. Six tasks are conducted for $\{30, 20, 20, 20, 30, 30\}$ epochs with batch size 16, learning rate $6\times 10^{-4}$ with linear warm-up over the first $6\%$ of steps and weight decay 0.01; for LoRA+, we set the learning rate to $2\times 10^{-4}$ and the learning-rate ratio to 8.

\textbf{Results.} Table~\ref{tab:superglue} summarizes validation accuracy across six SuperGLUE tasks. As shown in Table~\ref{tab:superglue}, SDLRT achieves the highest average accuracy of $85.15\%$, with gains of at least $0.34\%$.

\section{Discussion and Limitations}
\label{limits}
This work first derives the gradient-flow equations for the SVD factorization of full-parameter and low-rank network weights based on the manifold's geometric properties, and identifies that the difference between the two is related to the interaction between the small singular-value components and the large ones. Based on this, we propose a training method that enhances the stability of low-rank networks. The proposed method does not increase the rank of the network under a given tolerance but incorporates information from small singular values, thereby maintaining the accuracy of low-rank networks even under extreme compression rates. Moreover, the compensatory low-rank training framework proposed in this paper may provide insights for other low-rank training methods under different regularization settings. However, we note that the negative-feedback factor $\omega$ is set heuristically and scaling SDLRT to large architectures remains future work.

% \textbf{Limitations.} (i) The theoretical analysis assumes distinct singular values ($\sigma_i^2 \ne \sigma_{r+j}^2$); behavior near spectral degeneracies is not addressed. (ii) Our experiments focus on vision classification and NLU; scaling SDLRT to larger architectures (e.g., LLMs with billions of parameters) remains future work. (iii) The negative-feedback factor $\omega$ is set heuristically; a principled adaptive schedule would be desirable.

% ---- Acknowledgments: hidden automatically in submission version ----
\begin{ack}
Anonymized for review. Funding sources and competing interests will be declared here in the camera-ready version.
\end{ack}

% ---- References ----
% \bibliographystyle{ieeetr}
\bibliographystyle{unsrtnat}
\bibliography{reference}

\begin{thebibliography}{53}
\providecommand{\natexlab}[1]{#1}
\providecommand{\url}[1]{\texttt{#1}}
\expandafter\ifx\csname urlstyle\endcsname\relax
  \providecommand{\doi}[1]{doi: #1}\else
  \providecommand{\doi}{doi: \begingroup \urlstyle{rm}\Url}\fi

\bibitem[Gural et~al.(2021)Gural, Nadeau, Tikekar, and Murmann]{gural2021lowrank}
Albert Gural, Phillip Nadeau, Mehul Tikekar, and Boris Murmann.
\newblock Low-rank training of deep neural networks for emerging memory technology, 2021.
\newblock URL \url{https://arxiv.org/abs/2009.03887}.

\bibitem[Schotth{\"{o}}fer et~al.(2022)Schotth{\"{o}}fer, Zangrando, Kusch, Ceruti, and Tudisco]{schotthofer2022low}
Steffen Schotth{\"{o}}fer, Emanuele Zangrando, Jonas Kusch, Gianluca Ceruti, and Francesco Tudisco.
\newblock Low-rank lottery tickets: finding efficient low-rank neural networks via matrix differential equations.
\newblock In \emph{NeurIPS}, 2022.

\bibitem[Frankle and Carbin(2019)]{frankle2019lottery}
Jonathan Frankle and Michael Carbin.
\newblock The lottery ticket hypothesis: Finding sparse, trainable neural networks.
\newblock In \emph{{ICLR}}. OpenReview.net, 2019.

\bibitem[You et~al.(2022)You, Li, Sun, Ouyang, and Lin]{you2022supertickets}
Haoran You, Baopu Li, Zhanyi Sun, Xu~Ouyang, and Yingyan Lin.
\newblock Supertickets: Drawing task-agnostic lottery tickets from supernets via jointly architecture searching and parameter pruning.
\newblock In \emph{{ECCV} {(11)}}, Lecture Notes in Computer Science, pages 674--690. Springer, 2022.

\bibitem[Frankle et~al.(2020{\natexlab{a}})Frankle, Dziugaite, Roy, and Carbin]{frankle2020linear}
Jonathan Frankle, Gintare~Karolina Dziugaite, Daniel~M. Roy, and Michael Carbin.
\newblock Linear mode connectivity and the lottery ticket hypothesis.
\newblock In \emph{{ICML}}, Proceedings of Machine Learning Research, pages 3259--3269. {PMLR}, 2020{\natexlab{a}}.

\bibitem[Frankle et~al.(2020{\natexlab{b}})Frankle, Dziugaite, Roy, and Carbin]{frankle2019stabilizing}
Jonathan Frankle, Gintare~Karolina Dziugaite, Daniel~M. Roy, and Michael Carbin.
\newblock Stabilizing the lottery ticket hypothesis, 2020{\natexlab{b}}.
\newblock URL \url{https://arxiv.org/abs/1903.01611}.

\bibitem[Koch and Lubich(2007)]{koch2007dynamical}
Othmar Koch and Christian Lubich.
\newblock Dynamical low-rank approximation.
\newblock \emph{{SIAM} J. Matrix Anal. Appl.}, 29\penalty0 (2):\penalty0 434--454, 2007.

\bibitem[Han et~al.(2016)Han, Mao, and Dally]{han2016deep}
Song Han, Huizi Mao, and William~J. Dally.
\newblock Deep compression: Compressing deep neural networks with pruning, trained quantization and huffman coding, 2016.
\newblock URL \url{https://arxiv.org/abs/1510.00149}.

\bibitem[Frantar and Alistarh(2023)]{frantar2023sparsegpt}
Elias Frantar and Dan Alistarh.
\newblock Sparsegpt: Massive language models can be accurately pruned in one-shot.
\newblock In \emph{{ICML}}, Proceedings of Machine Learning Research, pages 10323--10337. {PMLR}, 2023.

\bibitem[Sun et~al.(2024)Sun, Liu, Bair, and Kolter]{sun2024wanda}
Mingjie Sun, Zhuang Liu, Anna Bair, and J.~Zico Kolter.
\newblock A simple and effective pruning approach for large language models.
\newblock In \emph{{ICLR}}. OpenReview.net, 2024.

\bibitem[Fang et~al.(2023)Fang, Ma, Song, Mi, and Wang]{fang2023depgraph}
Gongfan Fang, Xinyin Ma, Mingli Song, Michael~Bi Mi, and Xinchao Wang.
\newblock Depgraph: Towards any structural pruning.
\newblock In \emph{{CVPR}}, pages 16091--16101. {IEEE}, 2023.

\bibitem[Ma et~al.(2023)Ma, Fang, and Wang]{ma2023llmpruner}
Xinyin Ma, Gongfan Fang, and Xinchao Wang.
\newblock Llm-pruner: On the structural pruning of large language models.
\newblock In \emph{NeurIPS}, 2023.

\bibitem[Ashkboos et~al.(2024)Ashkboos, Croci, Nascimento, Hoefler, and Hensman]{ashkboos2024slicegpt}
Saleh Ashkboos, Maximilian~L. Croci, Marcelo Gennari~Do Nascimento, Torsten Hoefler, and James Hensman.
\newblock {SliceGPT}: Compress large language models by deleting rows and columns.
\newblock In \emph{{ICLR}}. OpenReview.net, 2024.

\bibitem[He and Xiao(2024)]{he2024structured}
Yang He and Lingao Xiao.
\newblock Structured pruning for deep convolutional neural networks: {A} survey.
\newblock \emph{{IEEE} Trans. Pattern Anal. Mach. Intell.}, 46\penalty0 (5):\penalty0 2900--2919, 2024.

\bibitem[Liu et~al.(2019)Liu, Sun, Zhou, Huang, and Darrell]{liu2019rethinking}
Zhuang Liu, Mingjie Sun, Tinghui Zhou, Gao Huang, and Trevor Darrell.
\newblock Rethinking the value of network pruning.
\newblock In \emph{{ICLR} (Poster)}. OpenReview.net, 2019.

\bibitem[Renda et~al.(2020)Renda, Frankle, and Carbin]{renda2020comparing}
Alex Renda, Jonathan Frankle, and Michael Carbin.
\newblock Comparing rewinding and fine-tuning in neural network pruning.
\newblock In \emph{{ICLR}}. OpenReview.net, 2020.

\bibitem[Cheng et~al.(2024)Cheng, Zhang, and Shi]{cheng2024survey}
Hongrong Cheng, Miao Zhang, and Javen~Qinfeng Shi.
\newblock A survey on deep neural network pruning: Taxonomy, comparison, analysis, and recommendations.
\newblock \emph{{IEEE} Trans. Pattern Anal. Mach. Intell.}, 46\penalty0 (12):\penalty0 10558--10578, 2024.

\bibitem[Li et~al.(2023)Li, Yu, Zhang, Liang, He, Chen, and Zhao]{li2023losparse}
Yixiao Li, Yifan Yu, Qingru Zhang, Chen Liang, Pengcheng He, Weizhu Chen, and Tuo Zhao.
\newblock {LoSparse}: Structured compression of large language models based on low-rank and sparse approximation.
\newblock In \emph{{ICML}}, Proceedings of Machine Learning Research, pages 20336--20350. {PMLR}, 2023.

\bibitem[Zhang et~al.(2024)Zhang, Chen, Shen, Yang, Ou, Yu, and Zhuang]{zhang2023loraprune}
Mingyang Zhang, Hao Chen, Chunhua Shen, Zhen Yang, Linlin Ou, Xinyi Yu, and Bohan Zhuang.
\newblock {LoRAPrune}: Structured pruning meets low-rank parameter-efficient fine-tuning.
\newblock In \emph{{ACL} (Findings)}, Findings of {ACL}, pages 3013--3026. Association for Computational Linguistics, 2024.

\bibitem[Idelbayev and Carreira{-}Perpi{\~{n}}{\'{a}}n(2020)]{idelbayev2020lowrank}
Yerlan Idelbayev and Miguel~{\'{A}}. Carreira{-}Perpi{\~{n}}{\'{a}}n.
\newblock Low-rank compression of neural nets: Learning the rank of each layer.
\newblock In \emph{{CVPR}}, pages 8046--8056. Computer Vision Foundation / {IEEE}, 2020.

\bibitem[Morcos et~al.(2019)Morcos, Yu, Paganini, and Tian]{morcos2019one}
Ari~S. Morcos, Haonan Yu, Michela Paganini, and Yuandong Tian.
\newblock One ticket to win them all: generalizing lottery ticket initializations across datasets and optimizers.
\newblock In \emph{NeurIPS}, pages 4933--4943, 2019.

\bibitem[Chen et~al.(2021{\natexlab{a}})Chen, Frankle, Chang, Liu, Zhang, Carbin, and Wang]{chen2021lottery_cv}
Tianlong Chen, Jonathan Frankle, Shiyu Chang, Sijia Liu, Yang Zhang, Michael Carbin, and Zhangyang Wang.
\newblock The lottery tickets hypothesis for supervised and self-supervised pre-training in computer vision models.
\newblock In \emph{{CVPR}}, pages 16306--16316. Computer Vision Foundation / {IEEE}, 2021{\natexlab{a}}.

\bibitem[Chen et~al.(2020)Chen, Frankle, Chang, Liu, Zhang, Wang, and Carbin]{chen2020lottery_bert}
Tianlong Chen, Jonathan Frankle, Shiyu Chang, Sijia Liu, Yang Zhang, Zhangyang Wang, and Michael Carbin.
\newblock The lottery ticket hypothesis for pre-trained {BERT} networks.
\newblock In \emph{NeurIPS}, 2020.

\bibitem[Ramanujan et~al.(2020)Ramanujan, Wortsman, Kembhavi, Farhadi, and Rastegari]{ramanujan2020whats}
Vivek Ramanujan, Mitchell Wortsman, Aniruddha Kembhavi, Ali Farhadi, and Mohammad Rastegari.
\newblock What's hidden in a randomly weighted neural network?
\newblock In \emph{{CVPR}}, pages 11890--11899. Computer Vision Foundation / {IEEE}, 2020.

\bibitem[Malach et~al.(2020)Malach, Yehudai, Shalev{-}Shwartz, and Shamir]{malach2020proving}
Eran Malach, Gilad Yehudai, Shai Shalev{-}Shwartz, and Ohad Shamir.
\newblock Proving the lottery ticket hypothesis: Pruning is all you need.
\newblock In \emph{{ICML}}, Proceedings of Machine Learning Research, pages 6682--6691. {PMLR}, 2020.

\bibitem[Pensia et~al.(2020)Pensia, Rajput, Nagle, Vishwakarma, and Papailiopoulos]{pensia2020optimal}
Ankit Pensia, Shashank Rajput, Alliot Nagle, Harit Vishwakarma, and Dimitris~S. Papailiopoulos.
\newblock Optimal lottery tickets via subset sum: Logarithmic over-parameterization is sufficient.
\newblock In \emph{NeurIPS}, 2020.

\bibitem[da~Cunha et~al.(2022)da~Cunha, Natale, and Viennot]{dacunha2022proving}
Arthur da~Cunha, Emanuele Natale, and Laurent Viennot.
\newblock Proving the lottery ticket hypothesis for convolutional neural networks.
\newblock In \emph{{ICLR}}. OpenReview.net, 2022.

\bibitem[Chen et~al.(2021{\natexlab{b}})Chen, Sui, Chen, Zhang, and Wang]{chen2021unified}
Tianlong Chen, Yongduo Sui, Xuxi Chen, Aston Zhang, and Zhangyang Wang.
\newblock A unified lottery ticket hypothesis for graph neural networks.
\newblock In \emph{{ICML}}, Proceedings of Machine Learning Research, pages 1695--1706. {PMLR}, 2021{\natexlab{b}}.

\bibitem[Hui et~al.(2023)Hui, Yan, Ma, and Ku]{hui2023rethinking}
Bo~Hui, Da~Yan, Xiaolong Ma, and Wei{-}Shinn Ku.
\newblock Rethinking graph lottery tickets: Graph sparsity matters.
\newblock In \emph{{ICLR}}. OpenReview.net, 2023.

\bibitem[Chen et~al.(2021{\natexlab{c}})Chen, Zhang, Sui, and Chen]{chen2021gans}
Xuxi Chen, Zhenyu Zhang, Yongduo Sui, and Tianlong Chen.
\newblock Gans can play lottery tickets too.
\newblock In \emph{{ICLR}}. OpenReview.net, 2021{\natexlab{c}}.

\bibitem[Jiang et~al.(2023)Jiang, Hui, Liu, and Yan]{jiang2023successfully}
Chao Jiang, Bo~Hui, Bohan Liu, and Da~Yan.
\newblock Successfully applying lottery ticket hypothesis to diffusion model, 2023.
\newblock URL \url{https://arxiv.org/abs/2310.18823}.

\bibitem[Liu et~al.(2024)Liu, Zhang, He, Wang, Xiao, Ye, Zhou, Ku, and Hui]{liu2024survey_lth}
Bohan Liu, Zijie Zhang, Peixiong He, Zhensen Wang, Yang Xiao, Ruimeng Ye, Yang Zhou, Wei-Shinn Ku, and Bo~Hui.
\newblock A survey of lottery ticket hypothesis, 2024.
\newblock URL \url{https://arxiv.org/abs/2403.04861}.

\bibitem[Paul et~al.(2023)Paul, Chen, Larsen, Frankle, Ganguli, and Dziugaite]{paul2023unmasking}
Mansheej Paul, Feng Chen, Brett~W. Larsen, Jonathan Frankle, Surya Ganguli, and Gintare~Karolina Dziugaite.
\newblock Unmasking the lottery ticket hypothesis: What's encoded in a winning ticket's mask?
\newblock In \emph{{ICLR}}. OpenReview.net, 2023.

\bibitem[Zhang et~al.(2021)Zhang, Jin, Zhang, Zhou, Zhao, Ren, Liu, Wu, Jin, and Dou]{zhang2021validating}
Zeru Zhang, Jiayin Jin, Zijie Zhang, Yang Zhou, Xin Zhao, Jiaxiang Ren, Ji~Liu, Lingfei Wu, Ruoming Jin, and Dejing Dou.
\newblock Validating the lottery ticket hypothesis with inertial manifold theory.
\newblock In \emph{NeurIPS}, pages 30196--30210, 2021.

\bibitem[Ma et~al.(2021)Ma, Yuan, Shen, Chen, Chen, Chen, Liu, Qin, Liu, Wang, and Wang]{ma2021sanity}
Xiaolong Ma, Geng Yuan, Xuan Shen, Tianlong Chen, Xuxi Chen, Xiaohan Chen, Ning Liu, Minghai Qin, Sijia Liu, Zhangyang Wang, and Yanzhi Wang.
\newblock Sanity checks for lottery tickets: Does your winning ticket really win the jackpot?
\newblock In \emph{NeurIPS}, pages 12749--12760, 2021.

\bibitem[Burkholz(2022)]{burkholz2022convolutional}
Rebekka Burkholz.
\newblock Convolutional and residual networks provably contain lottery tickets.
\newblock In \emph{{ICML}}, Proceedings of Machine Learning Research, pages 2414--2433. {PMLR}, 2022.

\bibitem[Ceruti and Lubich(2022)]{ceruti2022unconventional}
Gianluca Ceruti and Christian Lubich.
\newblock An unconventional robust integrator for dynamical low-rank approximation.
\newblock \emph{BIT Numerical Mathematics}, 62:\penalty0 23--44, 2022.

\bibitem[Ceruti et~al.(2022)Ceruti, Kusch, and Lubich]{ceruti2022rankadaptive}
Gianluca Ceruti, Jonas Kusch, and Christian Lubich.
\newblock A rank-adaptive robust integrator for dynamical low-rank approximation.
\newblock \emph{BIT Numerical Mathematics}, 62:\penalty0 1149--1174, 2022.

\bibitem[Zangrando et~al.(2024)Zangrando, Schotth{\"{o}}fer, Ceruti, Kusch, and Tudisco]{zangrando2024rankadaptive}
Emanuele Zangrando, Steffen Schotth{\"{o}}fer, Gianluca Ceruti, Jonas Kusch, and Francesco Tudisco.
\newblock Geometry-aware training of factorized layers in tensor tucker format.
\newblock In \emph{NeurIPS}, 2024.

\bibitem[Arora et~al.(2019)Arora, Cohen, Hu, and Luo]{arora2019implicit}
Sanjeev Arora, Nadav Cohen, Wei Hu, and Yuping Luo.
\newblock Implicit regularization in deep matrix factorization.
\newblock In \emph{NeurIPS}, pages 7411--7422, 2019.

\bibitem[Galanti et~al.(2024)Galanti, Siegel, Gupte, and Poggio]{galanti2024sgd}
Tomer Galanti, Zachary~S. Siegel, Aparna Gupte, and Tomaso Poggio.
\newblock {SGD} and weight decay secretly minimize the rank of your neural network, 2024.
\newblock URL \url{https://arxiv.org/abs/2206.05794}.

\bibitem[Yang et~al.(2020)Yang, Tang, Wen, Yan, Hu, Li, Li, and Chen]{yang2020learning}
Huanrui Yang, Minxue Tang, Wei Wen, Feng Yan, Daniel Hu, Ang Li, Hai Li, and Yiran Chen.
\newblock Learning low-rank deep neural networks via singular vector orthogonality regularization and singular value sparsification.
\newblock In \emph{{CVPR} Workshops}, pages 2899--2908. Computer Vision Foundation / {IEEE}, 2020.

\bibitem[Zhang et~al.(2023)Zhang, Chen, Bukharin, He, Cheng, Chen, and Zhao]{zhang2023adalora}
Qingru Zhang, Minshuo Chen, Alexander Bukharin, Pengcheng He, Yu~Cheng, Weizhu Chen, and Tuo Zhao.
\newblock Adaptive budget allocation for parameter-efficient fine-tuning.
\newblock In \emph{{ICLR}}. OpenReview.net, 2023.

\bibitem[Khodak et~al.(2021)Khodak, Tenenholtz, Mackey, and Fusi]{khodak2021init}
Mikhail Khodak, Neil~A. Tenenholtz, Lester Mackey, and Nicol{\`{o}} Fusi.
\newblock Initialization and regularization of factorized neural layers.
\newblock In \emph{{ICLR}}. OpenReview.net, 2021.

\bibitem[Razin et~al.(2022)Razin, Maman, and Cohen]{razin2022implicit}
Noam Razin, Asaf Maman, and Nadav Cohen.
\newblock Implicit regularization in hierarchical tensor factorization and deep convolutional neural networks.
\newblock In \emph{{ICML}}, Proceedings of Machine Learning Research, pages 18422--18462. {PMLR}, 2022.

\bibitem[Feppon and Lermusiaux(2018)]{feppon2018geometric}
Florian Feppon and Pierre F.~J. Lermusiaux.
\newblock A geometric approach to dynamical model order reduction.
\newblock \emph{{SIAM} J. Matrix Anal. Appl.}, 39\penalty0 (1):\penalty0 510--538, 2018.

\bibitem[Zhao et~al.(2024)Zhao, Zhang, Chen, Wang, Anandkumar, and Tian]{zhao2024galore}
Jiawei Zhao, Zhenyu Zhang, Beidi Chen, Zhangyang Wang, Anima Anandkumar, and Yuandong Tian.
\newblock {GaLore}: Memory-efficient {LLM} training by gradient low-rank projection.
\newblock In \emph{{ICML}}, Proceedings of Machine Learning Research, pages 61121--61143. {PMLR} / OpenReview.net, 2024.

\bibitem[Savostianova et~al.(2023)Savostianova, Zangrando, Ceruti, and Tudisco]{savostianova2023robust}
Dayana Savostianova, Emanuele Zangrando, Gianluca Ceruti, and Francesco Tudisco.
\newblock Robust low-rank training via approximate orthonormal constraints.
\newblock In \emph{NeurIPS}, 2023.

\bibitem[He et~al.(2023)He, Gao, and Chen]{he2023deberta}
Pengcheng He, Jianfeng Gao, and Weizhu Chen.
\newblock {DeBERTaV3}: Improving {DeBERTa} using electra-style pre-training with gradient-disentangled embedding sharing.
\newblock In \emph{{ICLR}}. OpenReview.net, 2023.

\bibitem[Wang et~al.(2019)Wang, Pruksachatkun, Nangia, Singh, Michael, Hill, Levy, and Bowman]{wang2019superglue}
Alex Wang, Yada Pruksachatkun, Nikita Nangia, Amanpreet Singh, Julian Michael, Felix Hill, Omer Levy, and Samuel~R. Bowman.
\newblock Superglue: {A} stickier benchmark for general-purpose language understanding systems.
\newblock In \emph{NeurIPS}, pages 3261--3275, 2019.

\bibitem[Hu et~al.(2022)Hu, Shen, Wallis, Allen{-}Zhu, Li, Wang, Wang, and Chen]{hu2022lora}
Edward~J. Hu, Yelong Shen, Phillip Wallis, Zeyuan Allen{-}Zhu, Yuanzhi Li, Shean Wang, Lu~Wang, and Weizhu Chen.
\newblock {LoRA}: Low-rank adaptation of large language models.
\newblock In \emph{{ICLR}}. OpenReview.net, 2022.

\bibitem[Hayou et~al.(2024)Hayou, Ghosh, and Yu]{hayou2024loraplus}
Soufiane Hayou, Nikhil Ghosh, and Bin Yu.
\newblock Lora+: Efficient low rank adaptation of large models.
\newblock In \emph{{ICML}}, Proceedings of Machine Learning Research, pages 17783--17806. {PMLR} / OpenReview.net, 2024.

\bibitem[Hairer et~al.(1993)Hairer, N{\o}rsett, and Wanner]{hairer1993solving}
Ernst Hairer, Syvert~Paul N{\o}rsett, and Gerhard Wanner.
\newblock \emph{Solving Ordinary Differential Equations {I}: Nonstiff Problems}, volume~8 of \emph{Springer Series in Computational Mathematics}.
\newblock Springer-Verlag, Berlin, 2 edition, 1993.

\end{thebibliography}

% ---- Appendix (does not count towards 9-page limit) ----
\appendix

\section{Broader Impact Statement}
\label{broader statement}

This paper presents work whose main goal is to improve the stability and reliability of low-rank neural network training under aggressive compression, while maintaining mathematically sound guarantees of convergence and approximation. As in the majority of deep learning research, there are potential societal consequences of our work, none of which we feel must be specifically highlighted here. A positive societal impact is the reduced carbon emissions and energy consumption enabled by more efficient neural network training and inference through effective model compression. Regarding ethical aspects, we feel nothing has to be added.

\section{Proof of Theorem~\ref{thm:exactness} (Exactness)}
\label{app:exactness}

We show that, when the right-hand side $F(t, Y)=\dot{W}(t)$ corresponds to a solution $W(t)$ of exact rank $r$ on $[t_0, t_1]$, the SDLRT integrator reproduces $W(t_1)$ exactly. The argument adapts the exactness proof of the rank-adaptive integrator in~\cite{ceruti2022rankadaptive} to account for the additional compensation bases $U_{\mathrm{neg}}$ and $V_{\mathrm{neg}}$.

\begin{proof}
Let $W(t)=U(t)S(t)V(t)^{\top}$ with $\operatorname{rank}(W(t))=r$ for all $t\in[t_0, t_1]$, and set $Y_0=U_0 S_0 V_0^{\top}=W(t_0)$.

\paragraph{K-step and L-step.} The K-step ODE $\dot{K}=\dot{W}(t)\,V_0$ with $K(t_0)=U_0 S_0$ integrates to
\begin{equation}
\label{eq:K1}
K_1 = U_0 S_0 + (W(t_1)-W(t_0))V_0 = W(t_1)\,V_0.
\end{equation}
Analogously, $L_1 = W(t_1)^{\top} U_0$.

\paragraph{Compensated basis augmentation.} SDLRT forms the augmented matrices $(K_1, U_0, U_{\mathrm{neg}})$ and $(L_1, V_0, V_{\mathrm{neg}})$, and computes their QR factorizations to obtain orthonormal bases $\widehat{U}$ and $\widehat{V}$.

Since $W(t_1)=U(t_1)S(t_1)V(t_1)^{\top}$ with $S(t_1)$ non-singular and $V(t_1)^{\top}V_0$ invertible (by assumption), Equation~\eqref{eq:K1} gives $\operatorname{range}(K_1)=\operatorname{range}(U(t_1))=\operatorname{range}(W(t_1))$. Because $\operatorname{range}(W(t_1))\subseteq\operatorname{range}(K_1, U_0, U_{\mathrm{neg}})$, we have
\begin{equation}\label{eq:proj_U}
\widehat{U}\widehat{U}^{\top} W(t_1) = W(t_1).
\end{equation}
An identical argument yields $W(t_1)\widehat{V}\widehat{V}^{\top} = W(t_1)$. Note that the compensation columns $U_{\mathrm{neg}}$ (resp.\ $V_{\mathrm{neg}}$) only enlarge the column space of the augmented matrix; they therefore preserve inclusion~\eqref{eq:proj_U} and cannot compromise exactness.

\paragraph{S-step (Galerkin update).} Because $U_0$ is a sub-block of $(K_1, U_0, U_{\mathrm{neg}})$, we have $\widehat{U}\widehat{U}^{\top}U_0=U_0$ (likewise $\widehat{V}\widehat{V}^{\top}V_0=V_0$). Setting $\widehat{M}=\widehat{U}^{\top}U_0$ and $\widehat{N}=\widehat{V}^{\top}V_0$, the S-step initial value is
\begin{equation}
\widehat{S}(t_0) = \widehat{M}\,S_0\,\widehat{N}^{\top} = \widehat{U}^{\top}\,W(t_0)\,\widehat{V},
\end{equation}
so the augmented initial approximation satisfies $\widehat{U}\widehat{S}(t_0)\widehat{V}^{\top}=W(t_0)=Y_0$, confirming that the compensation does not alter the starting point.

Integrating the S-step ODE $\dot{\widehat{S}}=\widehat{U}^{\top}\dot{W}(t)\widehat{V}$ gives
\begin{equation}
\widehat{S}(t_1) = \widehat{U}^{\top}W(t_0)\widehat{V} + \widehat{U}^{\top}(W(t_1)-W(t_0))\widehat{V} = \widehat{U}^{\top}W(t_1)\widehat{V}.
\end{equation}
Combining with~\eqref{eq:proj_U} and its right counterpart:
\begin{equation}
\widehat{U}\,\widehat{S}(t_1)\,\widehat{V}^{\top} = \widehat{U}\widehat{U}^{\top}W(t_1)\widehat{V}\widehat{V}^{\top} = W(t_1).
\end{equation}

\paragraph{Truncation.} Since $\operatorname{rank}(W(t_1))=r$, the matrix $\widehat{S}(t_1)$ has exactly $r$ non-zero singular values. Choosing the tolerance $\vartheta$ ensures that truncation retains all $r$ components, yielding $Y_1 = U_1 S_1 V_1^{\top} = W(t_1)$.
\end{proof}

\section{Proof of Theorem~\ref{thm:error} (Error Bound)}
\label{app_error}

\paragraph{Setup and Assumptions.}
Let $W_f(t)$ denote the full-parameter weight evolved by the gradient flow
\[
\dot{W}_f(t)=\mathcal{F}(W_f(t)):=-\nabla_W\mathcal{L}(W_f(t)),\qquad W_f(0)\text{ given}.
\]
We make the following assumptions, identical to those of DLRT~\cite{schotthofer2022low}:
\begin{itemize}
  \item[(A1)] \emph{Boundedness and Lipschitz continuity.} There exist
        constants $B,L>0$ such that, for every $W,\widetilde{W}$,
        $\|\mathcal{F}(W)\|_F\le B$ and
        $\|\mathcal{F}(W)-\mathcal{F}(\widetilde W)\|_F\le L\|W-\widetilde W\|_F$.
  \item[(A2)] \emph{$\varepsilon$-closeness to the manifold.} For every
        $W\in\mathcal{M}_r$ in a neighborhood of $W_f(t)$ along the flow,
        $\|(I-\Pi_{\mathcal{T}(W)})\,\mathcal{F}(W)\|_F\le\varepsilon$.
  \item[(A3)] \emph{Initial projection error.} The SDLRT initialization
        satisfies $\|U(0)S(0)V(0)^\top-W_f(0)\|_F\le\delta$.
\end{itemize}

\begin{theorem}[Error Bound]\label{thm:err-bound-full}
Let $\widehat U_t,\widehat S_t,\widehat V_t$ denote the SDLRT iterates after
$n$ steps of size $\eta>0$, with singular-value truncation tolerance
$\vartheta>0$. Under assumptions (A1)--(A3), there exist constants
$c_0,c_1,c_2,c_3>0$ depending only on $L$, $B$ and the final time
$T=n\eta$ (in particular, independent of any singular value of the exact
or approximate solutions) such that
\[
\bigl\|\widehat U_t\widehat S_t\widehat V_t^{\!\top}-W_f(n\eta)\bigr\|_F
\;\le\;
c_0\,\delta+c_1\,\gamma\varepsilon+c_2\,\eta+c_3\,\vartheta/\eta.
\]
\end{theorem}

\begin{proof}
We adapt the four-lemma argument of DLRT to the SDLRT compensation
step. Throughout the proof, $A(t)$ denotes the exact local solution of
$\dot A=\mathcal F(A)$ on $[0,\eta]$ with starting value $Y_0=
U_0S_0V_0^{\!\top}$, and $A_1:=A(\eta)$.

\textbf{Step 1: Compensation gain.}
The augmented bases updated by DLRT and SDLRT respectively are:
\begin{equation}
\begin{cases}
\widetilde{U}\in \text{orth}(K_1,U),\quad \widetilde{V}\in \text{orth}(L_1,V) \\
\widehat{U}\in \text{orth}(K_1,U,U_{\text{neg}}),\quad \widehat{V}\in\text{orth}(L_1,V,V_{\text{neg}}).
\end{cases}
\end{equation}
Thus, there is a strict subspace inclusion relationship:
\begin{equation}
    \text{range}(\widetilde{U})\subset \text{range}(\widehat{U}),\quad \text{range}(\widetilde{V})\subset \text{range}(\widehat{V}).
\end{equation}
Therefore, regarding the monotonicity of the orthogonal projection (the error does not increase when projecting to a larger subspace), we have: $\exists \gamma\in(0.1)$
\begin{equation}
\label{eq:gamma-bound}
\begin{cases}
    \|(I-\widehat U\widehat U^{\!\top})A_1\|_F\le \gamma\|(I-\widetilde U\widetilde U^{\!\top})A_1\|_F\le\gamma\theta\\\|(I-\widehat V\widehat V^{\!\top})A_1\|_F\le \gamma\|(I-\widetilde V\widetilde V^{\!\top})A_1\|_F\le\gamma\theta
\end{cases}
\end{equation}
where $A_1$ is the time-$\eta$ value of the local exact flow with starting value $Y_0$, and $\theta$ is the DLRT projection bound
\[ \theta\;=\;BL\bigl(4e^{L\eta}+9\bigr)\eta^{2}+\bigl(3e^{L\eta}+4\bigr)\varepsilon\eta+e^{L\eta}\delta_{\text{loc}},\] with $\delta_{\text{loc}}$ the deviation between the localstarting iterate and the full-rank flow.
\textbf{Step 1: Two-sided compensated projection error.}
From \eqref{eq:gamma-bound} and the unit-norm property of orthogonal
projectors,
\begin{equation}
\label{step1}
\begin{array}{cl}
\bigl\|\widehat U\widehat U^{\!\top}A_1\widehat V\widehat V^{\!\top}-A_1\bigr\|_F
&\le\bigl\|(\widehat U\widehat U^{\!\top}A_1-A_1)\widehat V\widehat V^{\!\top}\bigr\|_F
+\bigl\|A_1\widehat V\widehat V^{\!\top}-A_1\bigr\|_F\\
&\le\|(I-\widehat U\widehat U^{\!\top})A_1\|_F\cdot\|\widehat V\widehat V^{\!\top}\|_2
   + \|(I-\widehat V\widehat V^{\!\top})A_1^{\!\top}\|_F\\
&\le 2\gamma\theta.
\end{array}\end{equation}
This is the SDLRT analogue of DLRT's Lemma~3 (Lemma~6.2 in the
appendix), with the bound \emph{strictly} reduced by the factor
$\gamma\in(0,1)$.
\textbf{Step 2: Auxiliary decomposition for the S-step.}
Define $\widetilde S(t):=\widehat U^{\!\top}A(t)\widehat V$ for
$t\in[0,\eta]$, and decompose
\[
A(t)\;=\;\widehat U\widetilde S(t)\widehat V^{\!\top}+\mathcal R(t),
\qquad
\mathcal R(t):=A(t)-\widehat U\widehat U^{\!\top}A(t)\widehat V\widehat V^{\!\top}.
\]
By (A1) and~\eqref{step1},
\begin{equation}
\label{step2_1}
 \|\mathcal R(t)\|_F
\le \|\mathcal R(t)-\mathcal R(\eta)\|_F+\|\mathcal R(\eta)\|_F
\le 2B\eta+2\gamma\theta.   
\end{equation}

Writing $\mathcal F(A(t))=\mathcal F(\widehat U\widetilde S(t)\widehat V^{\!\top})+\mathcal D(t)$
with defect $\mathcal D(t)$, the Lipschitz bound (A1) yields
\begin{equation}
\label{step2_2}
\|\mathcal D(t)\|_F\le L\,\|\mathcal R(t)\|_F\le 2L\bigl(B\eta+\gamma\theta\bigr).
\end{equation}
\textbf{Step 3: Galerkin S-step error via Gronwall.}
The SDLRT $S$-step ODE
\[
\dot{\widehat S}(t)=\widehat U^{\!\top}\mathcal F(\widehat U\widehat S(t)\widehat V^{\!\top})\widehat V,
\qquad
\widehat S(0)=\widehat U^{\!\top}Y_0\widehat V,
\]
and the auxiliary equation
\[
\dot{\widetilde S}(t)=\widehat U^{\!\top}\mathcal F(\widehat U\widetilde S(t)\widehat V^{\!\top})\widehat V
                   +\widehat U^{\!\top}\mathcal D(t)\widehat V,
\qquad
\widetilde S(0)=\widehat U^{\!\top}Y_0\widehat V,
\]
share the same initial value (since $A_0=Y_0$ locally). The Gronwall
inequality combined with~\eqref{step2_2} gives
\begin{equation}
\label{step3}
\bigl\|\widehat S(\eta)-\widehat U^{\!\top}A_1\widehat V\bigr\|_F
\le\int_0^\eta e^{L(\eta-s)}\|\mathcal D(s)\|_F\,ds
\le 2Le^{L\eta}\bigl(B\eta+\gamma\theta\bigr)\eta. 
\end{equation}
\textbf{Step 4: Local error of one SDLRT step.}
Setting $Y_1:=\widehat U\widehat S(\eta)\widehat V^{\!\top}$ and combining~\eqref{step1} with~\eqref{step3},
\begin{equation}
\label{step4_1}
\begin{array}{cl}
\|Y_1-A_1\|_F
&\le \bigl\|\widehat U\widehat S(\eta)\widehat V^{\!\top}-\widehat U\widehat U^{\!\top}A_1\widehat V\widehat V^{\!\top}\bigr\|_F
   +\bigl\|\widehat U\widehat U^{\!\top}A_1\widehat V\widehat V^{\!\top}-A_1\bigr\|_F\\
&\le \bigl\|\widehat S(\eta)-\widehat U^{\!\top}A_1\widehat V\bigr\|_F+2\gamma\theta\\
&\le 2Le^{L\eta}B\eta^{2}+2Le^{L\eta}\gamma\theta\eta+2\gamma\theta.
\end{array}
\end{equation}
Using the explicit form of $\theta$ (with $\delta_{\text{loc}}=0$ for the per-step analysis),
\[
\gamma\theta=\gamma\bigl[BL(4e^{L\eta}+9)\eta^{2}+(3e^{L\eta}+4)\varepsilon\eta\bigr],
\]
and absorbing all $L,B$-dependent factors into constants
$\widehat c_1,\widehat c_2$ that depend only on $L,B,\eta$, the bound~\eqref{step4_1} simplifies to
\begin{equation}
\label{step4_2}
\|Y_1-A_1\|_F \;\le\; \widehat c_1\,\gamma\varepsilon\,\eta+\widehat c_2\,\eta^{2}.   
\end{equation}
\textbf{Step 5: Truncation step.}
After the rank-adaptive truncation in SDLRT, the discarded tail
contributes at most $\vartheta$ in Frobenius norm. Writing
$Y_1^{\mathrm{trunc}}=\widehat U_1\widehat S_1\widehat V_1^{\!\top}$ and
applying the triangle inequality to~\eqref{step4_2},
\begin{equation}
\label{step5}
\|Y_1^{\mathrm{trunc}}-A_1\|_F
\;\le\;\widehat c_1\,\gamma\varepsilon\,\eta+\widehat c_2\,\eta^{2}+\vartheta.
\end{equation}
\textbf{Step 6: Global error via Lady Windermere's fan.}
By Assumption (A1), the flow map of $\dot W=\mathcal F(W)$ is
$e^{L\eta}$-Lipschitz over each step. Applying the standard
error-propagation argument~\cite[Sec.\ II.3]{hairer1993solving} to the per-step error~\eqref{step5} over
$n$ steps with $T=n\eta$,
\begin{equation}
\begin{aligned}
\bigl\|\widehat U_t\widehat S_t\widehat V_t^{\!\top}-W_f(n\eta)\bigr\|_F
&\le e^{LT}\,\delta
   +\sum_{k=0}^{n-1}e^{L(n-k-1)\eta}\bigl(\widehat c_1\gamma\varepsilon\,\eta+\widehat c_2\,\eta^{2}+\vartheta\bigr)\\
&\le e^{LT}\,\delta
   +n\,e^{LT}\bigl(\widehat c_1\gamma\varepsilon\,\eta+\widehat c_2\,\eta^{2}+\vartheta\bigr)\\
&= e^{LT}\,\delta
   +e^{LT}\widehat c_1\,T\,\gamma\varepsilon
   +e^{LT}\widehat c_2\,T\,\eta
   +e^{LT}\,T\,\vartheta/\eta.
\end{aligned}
\end{equation}
Setting
\[
c_0=e^{LT},\qquad
c_1=e^{LT}\widehat c_1\,T,\qquad
c_2=e^{LT}\widehat c_2\,T,\qquad
c_3=e^{LT}\,T,
\]
we obtain
\[
\bigl\|\widehat U_t\widehat S_t\widehat V_t^{\!\top}-W_f(n\eta)\bigr\|_F
\;\le\;
c_0\,\delta+c_1\,\gamma\varepsilon+c_2\,\eta+c_3\,\vartheta/\eta,
\]
where $c_0,c_1,c_2,c_3$ depend only on $L,B,T$ and are independent of
the singular values of any iterate. \qed
\end{proof}

\section{Additional experiments}
\subsection{Additional experiments for the distribution of Table~\ref{tab:superglue}}

Figure~\ref{fig_superglue} shows the distribution of the validation accuracy from Table~\ref{tab:superglue} as a supplement to the results in section~\ref{sec:exp_peft}.
\begin{figure}[t]
  \centering
  \includegraphics[width=0.95\linewidth]{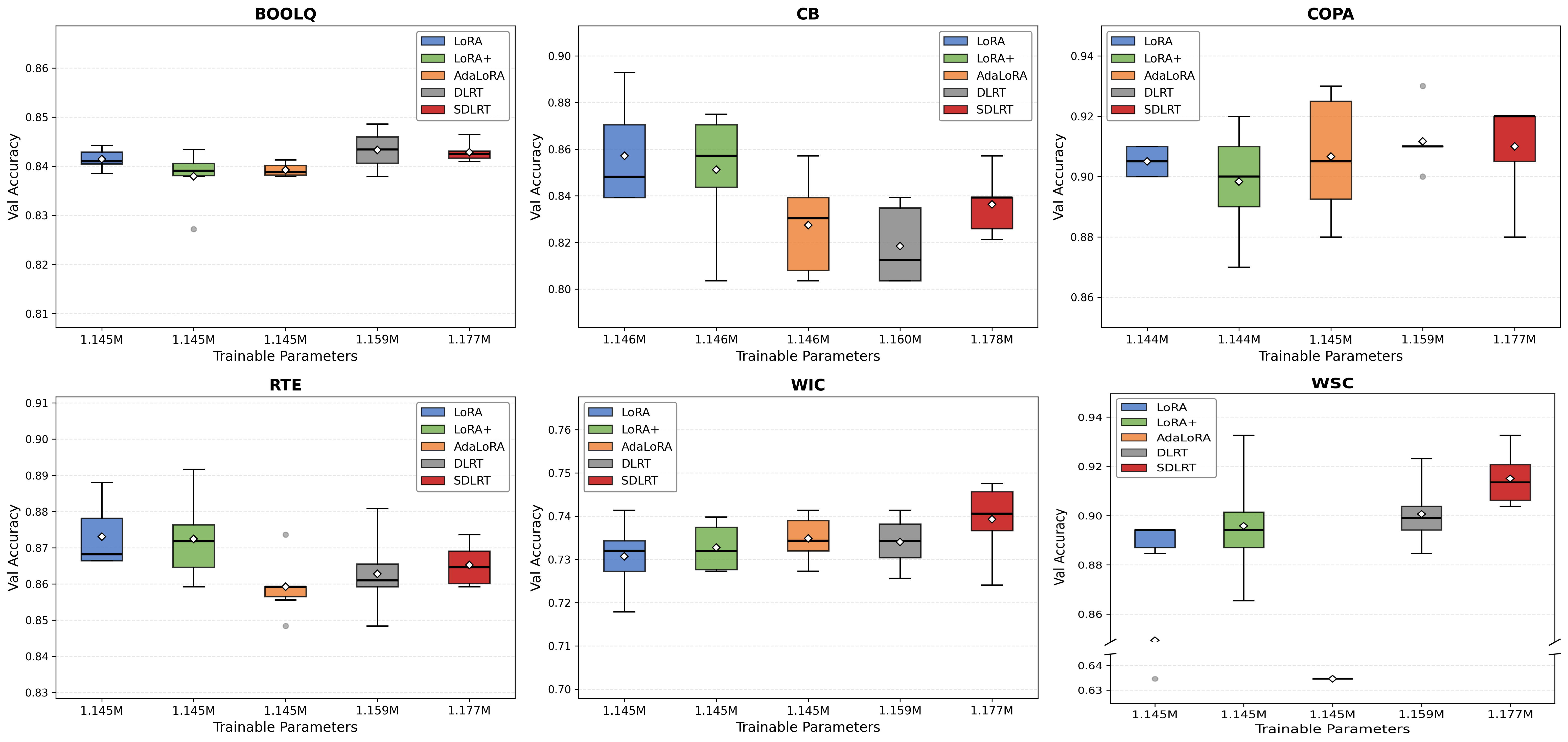}
  \caption{Distribution of the validation accuracy on DeBERTa-v3-base across six SuperGLUE tasks over six random seeds. The diamond marker indicates the mean and the circular marker indicates outliers. The box shows the interquartile range ($25\%$--$75\%$ quantiles) and the median.}
  \label{fig_superglue}
\end{figure}
% \subsection{Additional experiments for the comparison of DLRT and SDLRT on SuperGLUE}

% In addition to the comparison shown in the main text, Figure~\ref{fig_comparison} presents the performance of dlrt and sdlrt as LoRA-like adapters in fine-tuning.

% \begin{figure}[t]
%   \centering
%   \includegraphics[width=0.85\linewidth]{boxplot_sdlrt_vs_dlrt.png}
%   \caption{Comparison of DLRT and SDLRT with DeBERTa-v3-base. SDLRT outperforms DLRT almost on every task}
%   \label{fig_comparison}
% \end{figure}
\subsection{Additional experiments for the memory and time cost}

In Table~\ref{tab:mem_time}, we report the memory and time usage on VGG-19/CIFAR-100 setting $\text{seed}=30$, $\tau=0.45$ and batch size 128. As in Section~\ref{sec:exp_cv}, the training is running on an NVIDIA RTX3080.

\begin{table}[t]
\centering
\caption{Comparison of memory and time usage.}
\label{tab:mem_time}
\resizebox{0.8\textwidth}{!}{%
\begin{tabular}{lccccccc}
\toprule
Method & Memory & Training Time (Average per epoch) & Test Accuracy$(\%)$ \\
\midrule
DLRT & $2680\text{MB}$ & $56.88$s & $53.87\%$\\
SDLRT & $2710\text{MB}$ & $57.53$s & $61.86\%$ \\
\bottomrule
\end{tabular}}
\end{table}

\end{document}